\documentclass{article}

\usepackage[main, final]{neurips_2026}

\usepackage[utf8]{inputenc} 
\usepackage[T1]{fontenc}    
\usepackage{hyperref}       
\usepackage{url}            
\usepackage{booktabs}       
\usepackage{amsfonts}       
\usepackage{nicefrac}       
\usepackage{microtype}      
\usepackage[table]{xcolor}  
\usepackage{wrapfig}
\usepackage{graphicx}
\usepackage{amsmath}
\usepackage{amsthm}
\usepackage{cleveref}
\usepackage{thmtools, thm-restate}
\usepackage{natbib}
\usepackage{multirow}

\usepackage{threeparttable}
\usepackage{subcaption}
\usepackage{float}
\usepackage{algorithm}
\usepackage{algpseudocode}
\usepackage{colortbl}
\usepackage{xspace}
\usepackage{enumitem}
\usepackage{makecell}
\newcommand{\method}{FlashMol\xspace}

\newtheorem{proposition}{Proposition}
\newtheorem{definition}{Definition}

\crefname{appendix}{appendix}{appendices}
\Crefname{appendix}{Appendix}{Appendices}

\title{FlashMol: High-Quality Molecule Generation in as Few as Four Steps}

\author{%
  Xinyuan Wei$^{1,2*}$ \quad
  Zian Li$^{1,3*}$ \quad
  Shaoheng Yan$^{1,2}$ \quad
  Cai Zhou$^{4}$ \quad
  Muhan Zhang$^{1,5\dagger}$ \\[4pt]
  $^1$Institute for Artificial Intelligence, Peking University \\
  $^2$Yuanpei College, Peking University \\
  $^3$School of Intelligence Science and Technology, Peking University \\
  $^4$Department of
Electrical Engineering and Computer Science, Massachusetts Institute of Technology \\
  $^5$State Key Laboratory of General Artificial Intelligence, Peking University\\[2pt]
}

\graphicspath{{figures/}}

\begin{document}

\maketitle
\begin{center}
  \url{https://sylvan-wesley.github.io/flashmol.github.io/}
\end{center}
\renewcommand{\thefootnote}{\fnsymbol{footnote}}
\footnotetext{$^*$ Equal Contribution.}
\footnotetext{$\dagger$ Correspondence to Muhan Zhang <muhan@pku.edu.cn>.}
\renewcommand{\thefootnote}{\arabic{footnote}}


\begin{abstract}

Generating chemically valid 3D molecular conformations is critical for computational drug discovery. Classical diffusion-based models like GeoLDM perform well but require hundreds of steps, making large-scale in silico screening impractical. Recent efforts on few-step molecular generation have accelerated this process to 12-50 steps, but they often largely sacrifice sample stability. In this work, we present FlashMol, an ultra-fast molecule generative model producing high-quality molecular conformations in as few as 4 steps. To achieve this, we adapt distribution matching distillation (DMD) — a reverse KL-divergence minimization objective — to the molecular domain for effective distillation. Considering the local minimization behavior of DMD, we respace the molecule generation timesteps, providing the generator with much better initialization and enables effective distillation. Additionally, to mitigate the mode-seeking behavior of DMD and improve diversity, we further regularize it with a Jensen-Shannon divergence term, which incorporates the mean-seeking behavior of the forward KL divergence. Extensive experiments on QM9 and GEOM-DRUG datasets demonstrate that FlashMol matches and even surpasses the original 1000-step teacher: by applying respacing to the GeoLDM teacher (GeoLDM-respaced), we achieve comparable results to the 1000-NFE teacher with as few as 16 NFEs. After further distillation, FlashMol is able to sample with only 4 NFEs while maintaining high molecular quality, achieving up to 250x acceleration compared to the GeoLDM teacher.

\end{abstract}


\section{Introduction}

\label{section:intro}

Molecule generative models have emerged as a cornerstone of computational molecular design, offering
the prospect of exploring chemical space at scales that are infeasible for traditional physics-based
simulation~\citep{hoogeboom2022equivariantdiffusionmoleculegeneration, xu2023geometriclatentdiffusionmodels, li2025geometricrepresentationconditionimproves,irwin2025semlaflowefficient3d}. Among these approaches, diffusion-based molecule generative models such as EDM~\citep{hoogeboom2022equivariantdiffusionmoleculegeneration} and GeoLDM~\citep{xu2023geometriclatentdiffusionmodels} produce stable and valid molecules through iterative denoising, demonstrating strong potential in domains such as drug discovery~\citep{Warr2022ExplorationOU,
Kong2021MoleculeGB, Brylinski2014ComputationalRO}, protein engineering~\citep{Yildirim2024NextGenTP},
and materials science~\citep{Greenaway2021IntegratingCA}.

However, generating a molecule with diffusion-based methods typically requires hundreds to thousands
of sequential network evaluations.
For example, classic baselines including EDM~\citep{hoogeboom2022equivariantdiffusionmoleculegeneration} and GeoLDM~\citep{xu2023geometriclatentdiffusionmodels} require 1000 NFEs, and even more efficient flow-matching variants such as SemlaFlow~\citep{irwin2025semlaflowefficient3d} and
ET-Flow~\citep{hassan2024etflowequivariantflowmatchingmolecular} still require more than a hundred steps.
Although recent few-step generative methods~\citep{zhang2025accelerating, ni2025straightlinediffusionmodelefficient, li2025geometricrepresentationconditionimproves}
have reduced NFE to the range of 12--50 steps, they often come with substantially lower molecule stability, and no existing method, to the best of our knowledge, achieves reliable generation in 10 steps or fewer, a regime that is increasingly common in image and video domains~\citep{huang2025self, geng2025meanflowsonestepgenerative, song2023consistencymodels}. Sampling speed is therefore a major
bottleneck preventing these models from being deployed in the high-throughput screening pipelines essential for drug discovery.

\begin{figure}[h]
    \centering
    \includegraphics[width=0.8\textwidth]{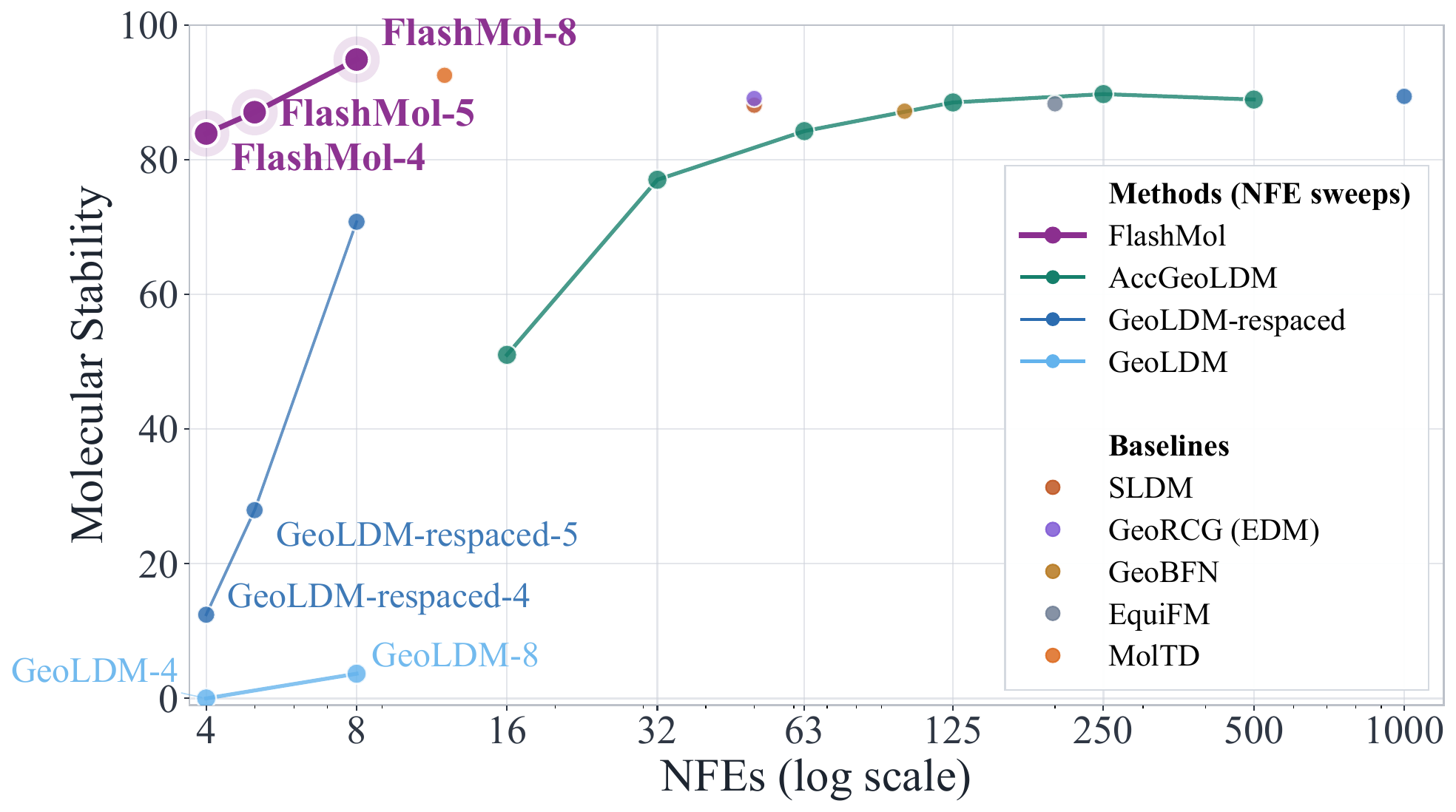}
    \caption{\textbf{Molecule generation quality on QM9.} We achieve comparable or superior performance on QM9 with only 4 to 8 steps, outperforming recent few-step molecule generative models. GeoLDM-respaced uses our respaced timesteps (\Cref{subsec:reschedule}); GeoLDM uses the original DDPM timesteps.}
    \label{fig:front}
\end{figure}

In this work, we propose \method, an ultra-fast molecule generative model that generates high-quality molecules in \textbf{as few as 4 steps}, as shown in~\Cref{fig:front}. To achieve this, we distill a few-step student from its 1000-step teacher via distribution matching distillation~\citep{yin2024onestepdiffusiondistributionmatching, yin2024improveddistributionmatchingdistillation}, a reverse-KL method that encourages high-fidelity samples and has shown strong promise in image and video domains~\citep{yin2024onestepdiffusiondistributionmatching, yin2024improveddistributionmatchingdistillation, jiang2026distributionmatchingdistillationmeets, huang2025self}. We identify two key obstacles to applying DMD in molecular generation. First, the local optimization behavior of DMD requires a well-initialized few-step student to provide useful optimization signals. Yet molecular structures are far less tolerant of generation errors than images: small deviations in bond lengths and angles can destroy chemical validity, making the initial few-step student much less likely to produce reasonable samples (nearly 0 $\%$ stability). To address this issue, we draw inspiration from the EDM design space analysis~\citep{karras2022elucidatingdesignspacediffusionbased}
and \textbf{redesign the discrete timestep schedule at inference and distillation training time}, without any retraining. This respacing substantially improves few-step quality and provides a viable student initialization for DMD. Second, the mode-seeking behavior of DMD gives the student a smaller probability support than the teacher, reducing molecular diversity. This issue is especially pronounced on QM9~\citep{ramakrishnan2014quantum}, which nearly enumerates the space of small molecules~\citep{xu2023geometriclatentdiffusionmodels}. To address this issue, we leverage a Jensen-Shannon divergence~\citep{xu2025onestepdiffusionmodelsfdivergence}, which \textbf{combines forward and reverse KL divergence, to regularize support shrinkage}, and treat it as a regularization term with a small coefficient, successfully improving the generation diversity.
To summarize, our contributions are as follows:
\begin{itemize}[noitemsep, topsep=0pt, left=0.3cm]
    \item We integrate distribution matching distillation into molecular generation using the latent molecular generative model GeoLDM~\citep{xu2023geometriclatentdiffusionmodels}.
    \item We identify inference-time timestep scheduling as a key bottleneck for effective distribution distillation and introduce a respacing strategy inspired
    by EDM~\citep{karras2022elucidatingdesignspacediffusionbased} that substantially improves the initial few-step generator even without training, and enables effective distillation.
    \item We extend the vanilla DMD objective with a Jensen-Shannon divergence term and provide a unified integration of KL-divergence-related objectives in molecular settings.
    \item We achieve \textbf{83.89\% molecular stability} on QM9 and \textbf{84.72\% atom stability} on GEOM-DRUG with only \textbf{4 steps}.
\end{itemize}

\section{Related Work}

\paragraph{Diffusion and Flow Models for 3D Molecule Generation}

Advances in molecular generation can be broadly categorized based on the output modality: 2D graph generation and 3D coordinate generation. Early work on molecular generation primarily focused on 2D graph generation~\citep{vignac2022digress,qin2024defog,latentgraphdiffusion}, where molecules are represented as graphs of atoms and bonds~\citep{jang2023hierarchical}. However, recent research has shifted towards generating 3D coordinates of atoms, allowing models to capture the spatial arrangement of molecules, which is crucial for their chemical properties. Prominent models such as the Equivariant Diffusion Model (EDM)~\citep{hoogeboom2022equivariantdiffusionmoleculegeneration} and GeoLDM~\citep{xu2023geometriclatentdiffusionmodels} have demonstrated impressive performance in both conditional and unconditional molecule generation tasks by adopting diffusion models. Despite their successes, these models often require a thousand function evaluations (NFEs) to generate a single sample. To address this limitation, flow-based models~\citep{irwin2025semlaflowefficient3d, song2023equivariantflowmatchinghybrid, dunn2026flowmol3} reduce the number of NFEs required, yet they still typically require over a hundred evaluations. Other methods, such as GeoRCG~\citep{li2025geometricrepresentationconditionimproves} and REED~\citep{wang2025learning}, use molecular representations to enhance generation and can generate stable molecules in as few as 50 steps. A recent method~\citep{zhang2025accelerating} introduces molecule-specific trajectory diagnosis, which accelerates sampling and achieves strong performance with as few as 12 steps.

\paragraph{Distillation and Distribution Matching for Accelerated Sampling}

Efforts to accelerate diffusion models generally focus on two strategies. The first involves progressive distillation~\citep{salimans2022progressivedistillationfastsampling}, consistency models~\citep{song2023consistencymodels,song2023improved}, shortcut models~\citep{lu2024simplifying, geng2025meanflowsonestepgenerative, lin2025design,zhou2025terminal}, and flow map distillation~\citep{boffi2025build,tong2025flow}, where a student model is trained to replicate a teacher model's denoising trajectory. While this approach reduces the required steps, it suffers from error accumulation as the trajectory length increases. The second strategy, distribution matching distillation~\citep{yin2024onestepdiffusiondistributionmatching, yin2024improveddistributionmatchingdistillation, jiang2026distributionmatchingdistillationmeets, xu2025onestepdiffusionmodelsfdivergence,tong2025flow}, uses a reverse KL-divergence objective to match the output distributions of the teacher and the student models. The reverse KL-divergence objective leads to mode-seeking behavior and preserves the sharpness of main modes, achieving high-quality individual samples~\cite{yin2024improveddistributionmatchingdistillation}. This line of distribution matching work has led to state-of-the-art performance in one-step generation and, since molecules are less tolerant to noise, we identify distribution matching objective as particularly suited to the molecule domain.

\section{Preliminaries}
\label{section:preliminaries}

\subsection{Molecule Diffusion Models}

In this paper, a molecule with $N$ atoms is represented as a point cloud $\mathcal{G}=\langle x,h\rangle$, where $x\in\mathbb{R}^{N\times 3}$ denotes 3D coordinates and $h\in\mathbb{R}^{N\times d}$ denotes atom features. \method adopts GeoLDM~\citep{xu2023geometriclatentdiffusionmodels} as its backbone due to its simplicity, which encodes $\mathcal{G}$ into a continuous latent molecule $z_0=\langle z_x,z_h\rangle$ and decodes the final latent sample back to molecular coordinates and features. The coordinate latent $z_x$ is kept in the zero-center-of-mass subspace to preserve translation invariance.

The latent diffusion process corrupts $z_0$ using a predefined noise schedule $(\alpha_t,\sigma_t)$:
$
z_t = \alpha_t z_0 + \sigma_t \epsilon, \epsilon\sim\mathcal{N}(0,I).
$
The denoising networks are implemented with equivariant architectures~\citep{hoogeboom2022equivariantdiffusionmoleculegeneration}, so rotations of coordinate latents induce corresponding rotations in the predicted coordinates while atom-feature channels remain invariant. Our method keeps this GeoLDM latent-space formulation, but distills a few-step student generator to replace the expensive multi-step teacher. Additional details on the original molecule diffusion formulation are provided in~\Cref{appendix:molecule_diffusion}.

\subsection{Distribution Matching Distillation}

\label{subsec:dmd bkg}

Distribution matching distillation (DMD)~\citep{yin2024improveddistributionmatchingdistillation, yin2024onestepdiffusiondistributionmatching} aims to minimize the KL divergence between the few-step student model and the multi-step teacher model $\text{KL}(p_{\text{fake}} \| p_{\text{real}})$. Let $F(x,t) = \alpha_tx + \sigma_t \epsilon$ denote the corruption process for non-geometric objects $x$, with $\epsilon \sim \mathcal{N}(0, I)$. The corresponding DMD loss has a gradient equivalent to:
\begin{align}
\label{eq:dmd}
    \nabla\mathcal{L}_\text{DMD} &=
        \mathbb{E}_t\left(\nabla_\theta \text{KL}(p_{\text{fake},t} \| p_{\text{real},t}) \right) \\&=
        - \mathbb{E}_{t}\left(\int \big(s_{\text{real}}(F(G_{\theta}(z), t), t) - s_{\text{fake},\phi}(F(G_{\theta}(z), t), t)\big) \frac{dG_\theta(z)}{d\theta} \hspace{.5mm} dz\right).
\end{align}
Here, $G_\theta(z)$ denotes the few-step student generator, starting from noise $z \sim \mathcal{N}(0, I)$, whose generated samples $x$ follow $x \sim p_{\text{fake}}$. The distributions $p_{\text{fake},t}$ and $p_{\text{real},t}$ correspond to the diffusion distributions defined by the forward operator $F$ of the student and teacher models, respectively. The score functions~\citep{song2021scorebasedgenerativemodelingstochastic} are defined as
$
s_{\text{real}}(x, t) = \nabla \log p_{\text{real},t}(x)$ and $s_{\text{fake}, \phi}(x, t) = \nabla \log p_{\text{fake},t}(x)
$, where $s_{\text{fake},\phi}(x, t)$ serves as a proxy multi-step score for the few-step student generator and is updated using the diffusion loss computed on samples produced by $G_\theta$~\citep{yin2024onestepdiffusiondistributionmatching} while $s_{\text{real}}(x, t)$ is approximated by the teacher model and is frozen during the training process. 

The teacher model is parameterized as $\mu_{\text{real}}$ and $s_{\text{fake},\phi}$ as $\mu_{\text{fake}, \phi}$, which are both noise-prediction diffusion models~\citep{yin2024improveddistributionmatchingdistillation}; the denoising output $\epsilon_\phi$ of $\mu_{\text{fake}, \phi}$
  is converted to $s_{\text{fake},\phi}$ via the score-based SDE framework~\citep{song2021scorebasedgenerativemodelingstochastic}.Following DMD2~\citep{yin2024improveddistributionmatchingdistillation}, in each iteration, $\mu_{\text{fake},\phi}$ is first updated for 5 steps using the diffusion loss on the generated sample of the student model: 
\begin{align}
\label{eq:fake_update}
    \mathcal{L}_{\text{fake}} = \mathbb{E}_{x \sim p_{\text{fake}}, \epsilon\sim \mathcal{N}(0,I)} \big[ w(t)||\epsilon_{\phi}(F(x,t),t) - \epsilon||^2
    \big ]
\end{align} where 
$x$ is the sample generated by the student model $G_\theta$ and 
$\epsilon_{\phi}(F(x,t),t)$ is the output of $\mu_{\text{fake}}$. The student generator $G_{\theta}$ is updated by the gradient of the DMD loss $\mathcal{L}_{\text{DMD}}$ for one step afterwards.

As in DMD2~\citep{yin2024improveddistributionmatchingdistillation}, few-step generator samples in a consistency-sampling manner~\citep{song2023consistencymodels}. Given a generator $G$, the final samples are obtained by repeatedly applying:
\begin{align}
\label{eq:cons_sample}
   x^{t+1} = G(z^t), \quad 
    z^{t+1} = \alpha_{t+1} x^{t+1} + \sigma_{t+1}\epsilon
\end{align}
where $\epsilon,z^0 \sim \mathcal{N}(0, I)$ and $\alpha_t, \sigma_t$ are fixed noise schedules. This iterative refinement significantly improves the sample quality of the alternative one-step generator~\citep{yin2024improveddistributionmatchingdistillation}.

\subsection{$f$-Divergences for Distribution Matching Distillation}
\label{subsection:f-div bkg}

DMD optimizes the reverse KL loss between student and teacher distributions. More generally, one can consider $f$-divergences of the form $D_f(p ||q) = \int q(x)f(\frac{p(x)}{q(x)})\space dx$. As derived in Xu et al.~\citep{xu2025onestepdiffusionmodelsfdivergence},
\begin{align}
\label{eq:f-div}
 \nabla_\theta D_f(p_{\text{real}}||p_{\text{fake}}) = \mathbb{E}_{\epsilon, t} \big [ - \int f''(r)r^2 \big(s_{\text{real}}(F(G_{\theta}(z), t), t) - s_{\text{fake}, \phi}(F(G_{\theta}(z), t), t)\big) \frac{dG_\theta(z)}{d\theta} \hspace{.5mm} dz \big ] 
\end{align}
where $r = \frac{p_{\text{real}}(x)}{p_{\text{fake}}(x)}$. For the Jensen-Shannon divergence, $f(r) := r \log r  - (r+1)\log \frac{r+1}{2}$. To approximate $r$, we can introduce a discriminator $D$ trained with the GAN objective, which yields the approximation $r \approx D(x) / (1 - D(x))$.

\section{Methods}

\label{sec:methods}

\subsection{Geometric Distribution Matching Distillation}
\begin{figure}[t]
    \centering
    \includegraphics[width=0.99\textwidth]{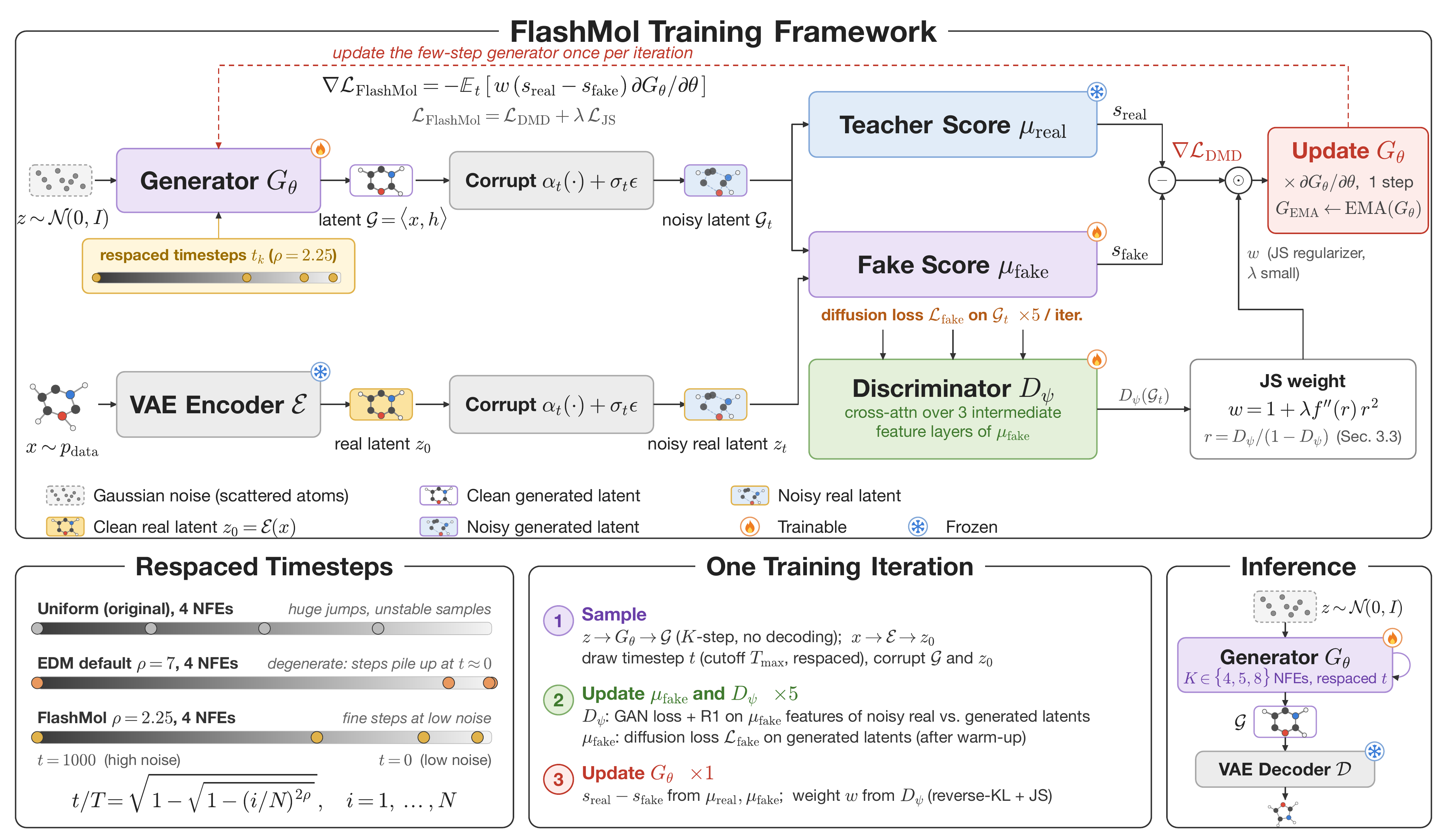}
    \caption{\textbf{Training framework for \method.} In each iteration, the few-step generator $G$ produces a batch of sampled molecules. Then, the $\mu_{\text{fake}}$ model which approximates fake score $s_{\text{fake}}$ and the discriminator $D$ are first updated using diffusion loss and GAN loss respectively for 5 steps. After that, $\mu_{\text{fake}}$ and the discriminator $D$'s output are used to compute the DMD and Jensen-Shannon divergence gradient. Lastly, the gradients are mixed and update the generator $G$.}
    \label{fig:framework}
\end{figure}

We now describe the key designs of the {\method} framework, which adapts the DMD approach to the molecular domain with carefully tailored mechanisms, achieving remarkable few-step generation performance. We choose GeoLDM~\citep{xu2023geometriclatentdiffusionmodels}, a strong latent molecule generative model, as the backbone of \method. The overall training framework is illustrated in~\Cref{fig:framework}. 
Following DMD~\citep{yin2024onestepdiffusiondistributionmatching}, our training uses three models: the few-step generator \(G_\theta\),
  the frozen multi-step teacher score estimator \(\mu_{\text{real}}\), and the proxy score estimator \(\mu_{\text{fake}}\) for the student
  generator. In each training iteration, the student generator $G_\theta$ first produces latent molecules $\mathcal{G}=\langle x,h\rangle$ from Gaussian noise $z$ via $K$-step consistency sampling, as defined in~\Cref{eq:cons_sample}, without decoding to the molecule space. During sampling, the coordinate component $x$ is always projected onto the zero-mean subspace to preserve translation invariance. To let $\mu_{\text{fake}}$ faithfully approximate the score of the student generator, we first update $\mu_{\text{fake}}$ according to the diffusion loss in~\Cref{eq:fake_update} for 5 steps following~\citep{yin2024improveddistributionmatchingdistillation}. Then, the generated latent molecules are passed through $\mu_{\text{real}}$ and $\mu_{\text{fake}}$ to estimate the scores $\nabla \log p_t(\langle x_t,h_t\rangle)$, which yield the DMD gradients following~\Cref{eq:dmd} for updating the student generator $G_\theta$. 
  
  Notably, in the geometric formulation of DMD, equivariance is always preserved: the resulting DMD gradient is invariant to rotation of the student samples. This property improves data efficiency in few-step distillation for equivariant molecular generators. For simplicity, we assume $G$ to be a one-step generator. Formally,
\begin{proposition}[invariance]
\label{thm:invariance}
    The distribution matching loss
$$\nabla\mathcal{L}_\text{DMD}(z;\theta) =
        - \mathbb{E}_{t}\left(\int \big(s_{\text{real}}(F(G_{\theta}(z), t), t) - s_{\text{fake}, \phi}(F(G_{\theta}(z), t), t)\big) \frac{dG_\theta(z)}{d\theta} \hspace{.5mm} dz\right)
$$ is \emph{invariant} to SE(3) group, where Gaussian noise $z=\langle z_x,z_h\rangle \in \mathbb{R}^{N \times (3+d)}$ and the group transformation applies only on $z_x$. Formal definitions and proof are in~\Cref{appendix:A}.
\end{proposition}

At inference, the few-step generator samples in consistency sampling manner with $K$-steps, and the sampled latent molecules are decoded using the VAE inherited from the pretrained GeoLDM model.

\subsection{Escaping Poor Initial Modes via Respaced Timesteps}
\label{subsec:reschedule}

Directly adapting DMD to molecules poses significant optimization challenges. In this subsection, we identify the key issue and show how the proposed timestep respacing strategy addresses it.

\paragraph{Local Optimization Behavior of DMD} Theoretically, the DMD loss is a complete loss in the sense that its optimal solution aligns the student and teacher distributions nearly perfectly, assuming unlimited model capacity. In practice, however, the reverse KL divergence,
$
\text{KL}(p_{\text{fake}} \| p_{\text{real}}) = \mathbb{E}_{x \sim p_{\text{fake}}} \big[ \log p_{\text{fake}}(x) - \log p_{\text{real}}(x) \big],
$
induces a local optimization behavior~\citep{jiang2026distributionmatchingdistillationmeets, bai2026optimizing, yin2025slow}, because \textbf{the samples are drawn from the student model}, rather than the teacher or data distribution as in the case of the forward KL.
Consequently, if the initial student model cannot generate reasonable (even if imperfect) samples, the DMD loss in~\Cref{eq:dmd} provides limited optimization signals. 

\begin{figure}[t]
\centering
\begin{minipage}[c]{0.55\textwidth}
\centering
\includegraphics[width=0.25\linewidth]{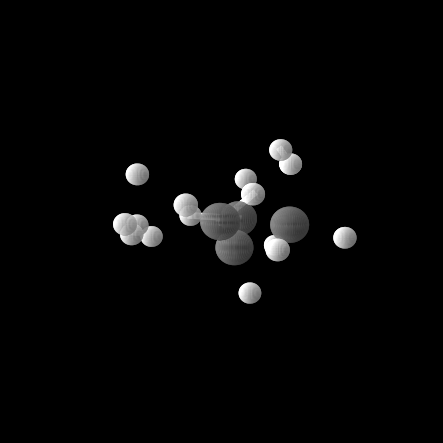}\hfill
\includegraphics[width=0.25\linewidth]{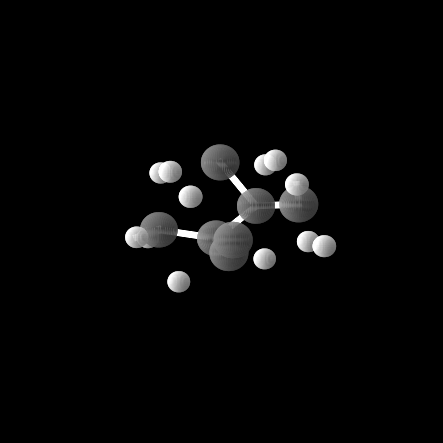}\hfill
\includegraphics[width=0.25\linewidth]{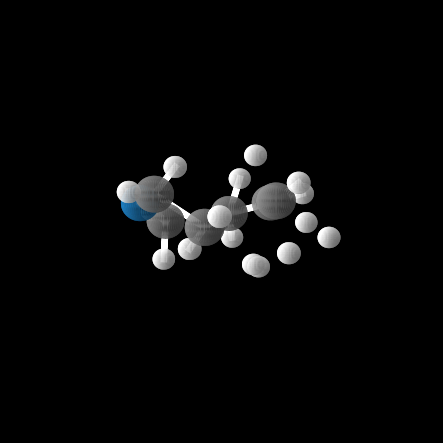}\hfill
\includegraphics[width=0.25\linewidth]{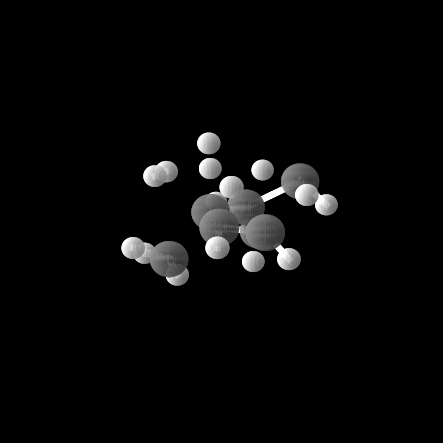}\\
\includegraphics[width=0.25\linewidth]{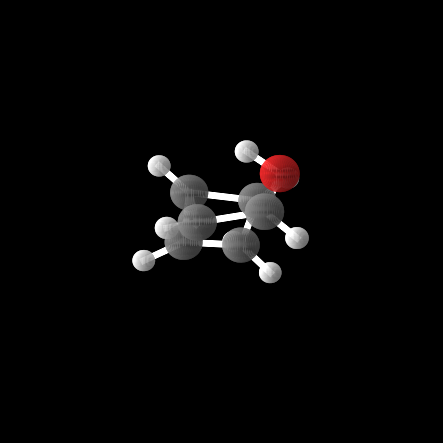}\hfill
\includegraphics[width=0.25\linewidth]{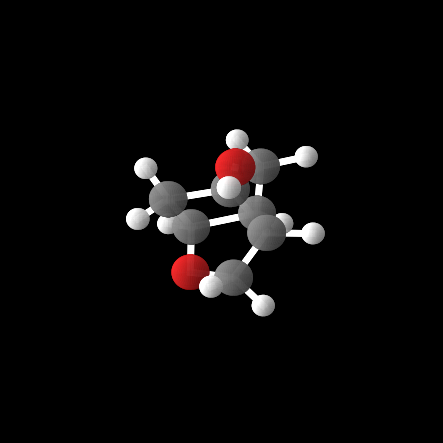}\hfill
\includegraphics[width=0.25\linewidth]{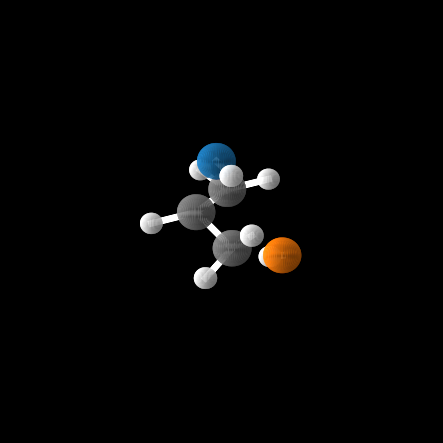}\hfill
\includegraphics[width=0.25\linewidth]{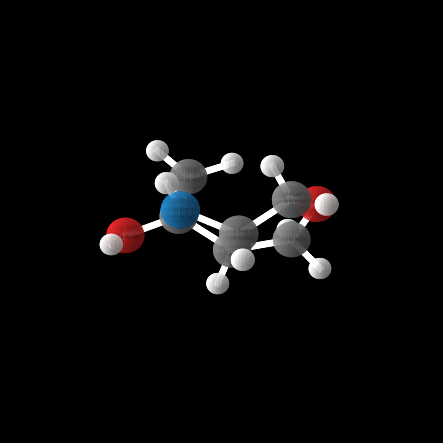}
\caption{Random 4-step samples from the multi-step teacher without additional training. Top: original GeoLDM schedule. Bottom: our respaced schedule.}
\label{fig:schedule_samples}
\end{minipage}
\hfill
\begin{minipage}[c]{0.40\textwidth}
\centering
\includegraphics[width=\linewidth]{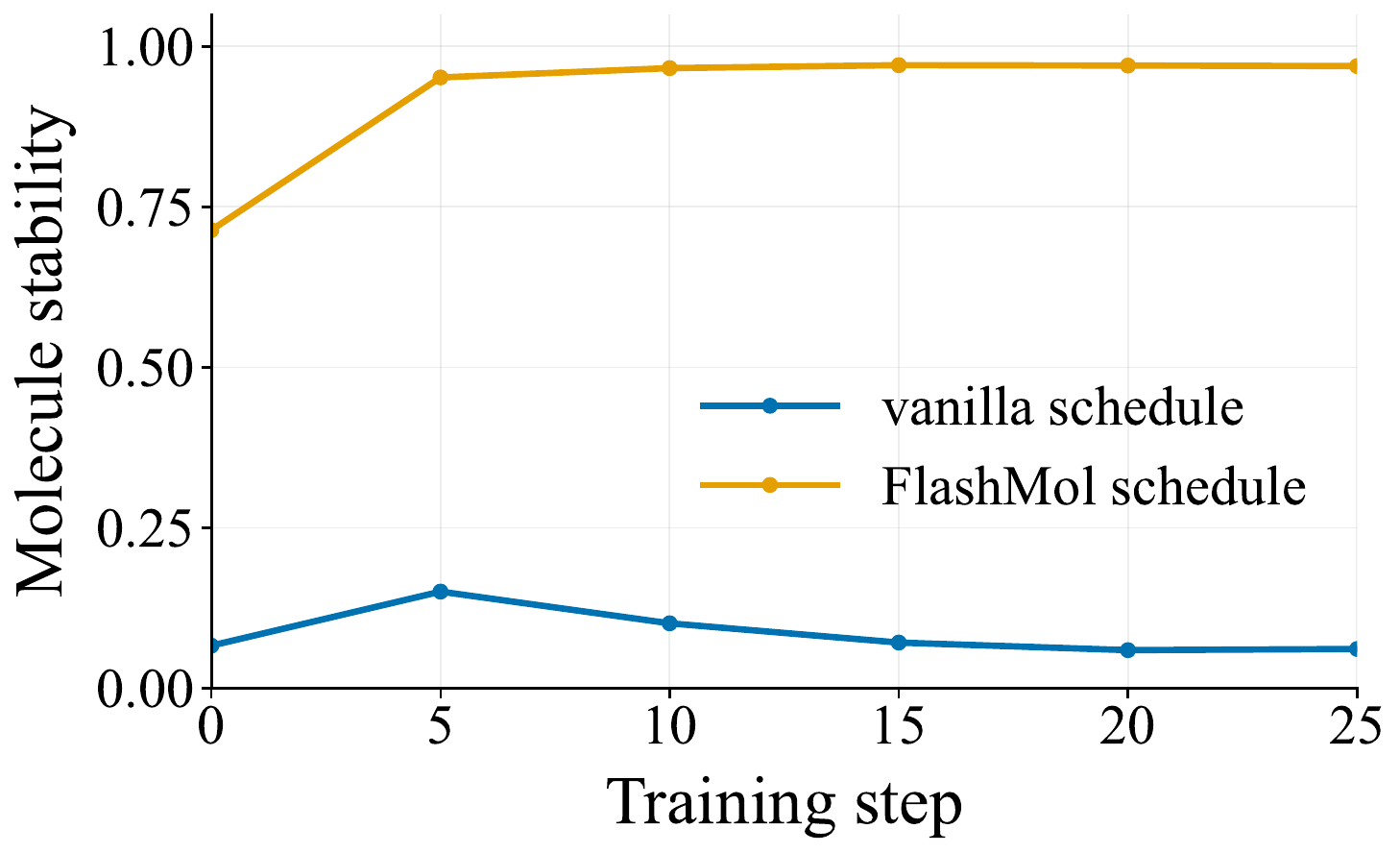}
\caption{Distribution matching distillation training dynamics under different sampling noise schedules.}
\label{fig:vanilla_training}
\end{minipage}
\end{figure}

Different from other modalities such as images, molecule generation exhibits two distinctive characteristics: (1) its process is highly sensitive to noise~\citep{ni2025straightlinediffusionmodelefficient}; and (2) molecular bond lengths and angles must be encoded with precise fine-grained detail, meaning that low-noise regions require careful handling. Directly initializing a few-step student model from a multi-step teacher model and
sampling with uniformly spaced timesteps~\cite{song2020denoising} makes huge jump in the denoising trajectory and introduces severe discretization error. On the other hand, higher NFEs allow the model to further refine its details in later steps~\cite{ho2020denoising}, leading to better quality. As a result, the few-step initial student produces poor-quality molecular samples.
As shown in the top row of~\Cref{fig:schedule_samples}, the 4-step student samples generated with the original timestep schedule are highly noisy and chemically invalid---they often contain disconnected atoms due to large deviations in atom-pair distances. Eventually, directly optimizing such a model with the DMD loss will lead to ineffective few-step distillation. In our initial experiments, we observe that this causes gradient explosions even before the first DMD training epoch is completed and results in poor-quality samples after training. As shown in~\Cref{fig:vanilla_training} (blue line), an 8-step student model sampled with uniform timesteps fails to achieve effective DMD optimization.

\paragraph{Respaced Timesteps for GeoLDM} To address this challenge, we devise a respaced timestep inspired by EDM~\citep{karras2022elucidatingdesignspacediffusionbased}, 
which shows that the optimal schedule, in the sense of minimized truncation error, for the noise levels $\{\sigma_i\}_{i=0}^{N-1}$ takes the form:
\begin{equation}
\label{eq:EDM}
\sigma_i = \Bigl(\sigma_{\max}^{1/\rho} + \frac{i}{N-1}
            \bigl(\sigma_{\min}^{1/\rho} - \sigma_{\max}^{1/\rho}\bigr)\Bigr)^{\!\rho},
\quad i = 0, \ldots, N-1.
\end{equation}
Crucially, this schedule
front-loads large denoising steps at high noise levels and reserves many small steps for
the low-noise regime, where fine geometric detail is recovered. 

However, GeoLDM inherits the variance preserving~\citep{song2021scorebasedgenerativemodelingstochastic} schedule
$\sigma_t = \sqrt{1 - \alpha_t^2}$ with $\alpha_t = 1-(t/T)^2$ and $T=1000$, which allocates timesteps roughly uniformly in noise-variance space. This conflicts with the EDM-style timestep.
We therefore select the timesteps according to the noise-level schedule in~\Cref{eq:EDM}. Specifically, we first choose noise levels uniformly from~\Cref{eq:EDM}, and then select the corresponding timesteps from the GeoLDM schedule to align with these noise levels.
Under this design, the discrete step $t \in [0, T]$ corresponding to the continuous-time noise level $\sigma_i$
satisfies $\sigma_i^2 = 1 - \alpha_t^2 = 1 - (1-(t/T)^2)^2$.
Inverting this relationship gives the formula:
\begin{equation}
\label{eq:reschedule}
\frac{t}{T} = \sqrt{1 - \sqrt{1 - (i/(N-1))^{2\rho}}}\,,
\end{equation}
and we apply this respaced schedule in both the few-step generator $G$ and the score estimator $\mu_{\text{real}}$ and $\mu_{\text{fake}}$, ensuring compatibility during training.\footnote{In practice, we adopt $\frac{t}{T} = \sqrt{1 - \sqrt{1 - (i/(N))^{2\rho}}}$ to avoid square-rooting a negative number. } 

Furthermore, we investigate the selection of $\rho$ in the schedule. We observe that the original configuration with $\rho=7$ in~\Cref{eq:reschedule} leads to the diffusion step $t \approx 0$ for more than 38\% of the 1000 denoising steps, which is a degenerate case and demonstrates poor sampling quality in molecule generation. This is visualized in~\Cref{fig:rho}. Thus, we experiment and obtain the optimal $\rho = 2.25$, which distinguishes between the $t$ values in $[0, 300]$ while remaining concentrated on low-noise regions. This is justified by our ablation experiments in~\Cref{subsec:rho}.

\paragraph{Benefits of Respacing} After timestep respacing, the generated molecules are dramatically more stable (bottom row of~\Cref{fig:schedule_samples}).        
Although not perfect before distillation, they provide a good starting point for escaping the poor initial modes in DMD training.
Consequently, DMD training is stabilized, with no significant gradient explosion even in later stages as shown in~\Cref{fig:vanilla_training} (orange line).

\subsection{Improving Diversity via $f$-Divergence Regularization}

DMD adopts a reverse KL-divergence objective which severely \textbf{penalizes the student generator for placing probability mass outside the teacher's support}, thereby damaging sampling diversity. 
By choosing the Jensen-Shannon divergence , 
the resulting objective $D_{\text{JS}}(p_{\text{real}}||p_{\text{fake}})$ can be interpreted as the mean of 
$\text{KL}\!\left(p_{\text{fake}}\,\|\,({p_{\text{fake}}+p_{\text{real}}}) /{2}\right)$ and 
$\text{KL}\!\left(p_{\text{real}}\,\|\,({p_{\text{fake}}+p_{\text{real}}}) /{2}\right)$, 
which incorporates both forward- and reverse-KL components~\citep{Lin1991DivergenceMB}. This adds mean-seeking behavior via the forward-KL component, counteracting DMD's reverse-KL mode seeking and improving sample diversity. In \method, we introduce this $f$-divergence as a regularization term complementary to the DMD loss, resulting in the final loss:
\begin{align}
\label{eq:final_loss}
   \mathcal{L}_{\text{\method}} = \mathcal{L}_{\text{DMD}} + \lambda \cdot \mathcal{L}_{\text{JS}}, 
\end{align}
where $\mathcal{L}_{\text{JS}}$ follows the definition in~\Cref{eq:f-div} with $f(r) := r \log r - (r+1)\log \frac{r+1}{2}$.
However, direct differentiation of the loss requires the computation of $\frac{p_{\text{fake}}(\mathcal{G})}{p_{\text{real}}(\mathcal{G})}$ given a molecule $\mathcal{G}$, which is intractable~\citep{xu2025onestepdiffusionmodelsfdivergence}. Therefore, we approximate this term using output of a GAN discriminator~\citep{goodfellow2014generative} according to the formula in~\Cref{subsection:f-div bkg}. 
Following DMD2~\citep{yin2024improveddistributionmatchingdistillation}, we train the discriminator on both the generated molecules of the few-step generator and real molecules in the dataset. To better capture both high-level semantics and low-level details, we extract three intermediate layers of features from the $\mu_{\text{fake}}$ network. Then, for each layer, we perform cross attention~\citep{vaswani2017attention} with a learnable token query for the real/fake logits. We leave the detailed implementation in~\Cref{appendix:gan_design} and the final algorithm in~\Cref{alg:training}.

\section{Experiments}

\label{section:experiments}

In this section, we show the remarkable performance of \method on various molecule generative tasks.
We first describe the experimental settings in~\Cref{subsection:setup}. Then, we report quantitative results and acceleration-quality trade-offs across different datasets and metrics in~\Cref{subsection:unconditional} and~\Cref{subsection:conditional}. Finally, we provide an ablation study on the key designs of \method, including time respacing and distillation loss, in~\Cref{subsec:ablation}.
\subsection{Experimental Settings} 
\label{subsection:setup}

\paragraph{Teacher Models}  For unconditional generation tasks, the teacher model is obtained from GeoLDM~\citep{xu2023geometriclatentdiffusionmodels}. For conditional generation, we train the conditional GeoLDM teacher from scratch. Implementation details can be found in~\Cref{appendix:B}. We also train another molecule generative model (SemlaFlow~\citep{irwin2025semlaflowefficient3d}) to demonstrate the transferability to another backbone and present the results in~\Cref{app:semla}.

\paragraph{Tasks and Metrics} We experiment with both unconditional and conditional generation following standard evaluation protocols~\citep{hoogeboom2022equivariantdiffusionmoleculegeneration, zhang2025accelerating}. Specifically, we evaluate \method on QM9~\citep{ramakrishnan2014quantum} dataset, which consists of 130K small molecules (with up to 29 atoms including hydrogens each), for both unconditional and conditional generation, and on GEOM-DRUG~\citep{axelrod2022geomenergyannotatedmolecularconformations} dataset, which consists of 450K larger molecules with an average atom number of 44 and maximum of 181, for unconditional generation. 
For unconditional generation, we report the generated molecules' 2D graph properties, including atom stability, molecule stability, validity and uniqueness following GeoLDM, SLDM, and MOLTD~\citep{xu2023geometriclatentdiffusionmodels, ni2025straightlinediffusionmodelefficient, zhang2025accelerating}. Since our models output atom types and atoms' 3D coordinates directly, 
bond types are computed via look-up tables according to pairwise distances and atom types~\citep{hoogeboom2022equivariantdiffusionmoleculegeneration}.
In addition to those metrics, we record PoseBusters metrics in~\Cref{tab:posebusters} following~\citep{buttenschoen2024posebusters, joshi2025all}, which evaluate 3D conformational properties, including bond length, bond angles and ring flatness. Furthermore, we report the number of function evaluations (NFE) to compare sampling efficiency across various models. For conditional generation, we report the mean absolute error (MAE) of the specific property between the generated molecule and the specified condition value.

\subsection{Unconditional Generation on QM9 and GEOM-DRUG}

\label{subsection:unconditional}

\begin{table*}[!t]
\centering
\caption{Unconditional molecule generation results of \method. We present models of three different step counts and our method achieves $250\times$ to $125\times$ acceleration as well as state-of-the-art performance on various datasets. We \colorbox{blue!15}{highlight} metrics in which our student model surpasses the respaced teacher using the same NFEs. \textsc{GeoLDM}-respaced results are obtained using our sampling method; see~\Cref{subsec:reschedule}. The MOLTD V\&U entry is estimated by us.}
\label{tab:main results}
\resizebox{0.85\textwidth}{!}{
\begin{threeparttable}
\begin{tabular}{l | c | c c c c | c c}
    \toprule[1.0pt]
    & & \multicolumn{4}{c|}{\shortstack[c]{\textbf{QM9}}} & \multicolumn{2}{c}{\shortstack[c]{\textbf{DRUG}}} \\
    \# Metrics & NFE & Atom Sta (\%) & Mol Sta (\%) & Valid (\%) & Valid \& Unique (\%) & Atom Sta (\%) & Valid (\%) \\
    \midrule[0.8pt]
    Data & - & 99.0 & 95.2 & 97.7 & 97.7 & 86.5 & 99.9 \\
    \midrule
    GDM-\textsc{aug} & 1000 & 97.6 & 71.6 & 90.4 & 89.5 & 77.7 & 91.8 \\
    EDM        & 1000 & 98.7 & 82.0 & 91.9 & 90.7 & 81.3 & 92.6 \\
    EDM-Bridge & 1000 & 98.8 & 84.6 & 92.0 & 90.7 & 82.4 & 92.8 \\
    GeoLDM    & 1000  & 98.9 & 89.4 & 93.8 & \textbf{92.7} & 84.4 & 99.3 \\
    EquiFM     & 200  & 98.9 & 88.3 & 94.7 & 93.5 & 84.1 & 98.9 \\
    GeoBFN     & 100  & 98.6 & 87.2 & 93.0 & 91.5 & 78.9 & 93.1 \\
    GOAT       & 90   & 99.2 & -     & 92.9 & 92.0 & 84.8 & 96.2 \\
    \midrule[0.3pt]
    SLDM & 50 & 98.70 & 88.09 & 92.84 & 89.53 & - & -
    \\
    GeoRCG (EDM) & 50 & 98.75 & 89.08 & 95.05 & - & 81.44 & 95.70 \\
    AccGeoLDM & 16 & 92.08 & 51.02 & 73.22 & 72.39 & - & - \\
    MOLTD & 12 & 99.40 & 92.53 & 96.04  & 92.0   & 86.88 & 95.33 \\
    \midrule[0.3pt]
        \textsc{GeoLDM}-respaced    &  8 & 95.49 $\pm$ 0.1 &  70.78 $\pm$ 0.4 & 87.13 $\pm$ 0.3 & 81.48 $\pm$ 0.4 & 64.35 & 90.44 \\
    \textbf{\textsc{\method}} & \textbf{8} & \cellcolor{blue!15}\textbf{99.55} $\pm$ 0.0 & \cellcolor{blue!15}\textbf{94.87}  $\pm$ 0.5 & \cellcolor{blue!15}\textbf{96.97} $\pm$ 0.2 & \cellcolor{blue!15}87.51 $\pm$ 0.4 & \cellcolor{blue!15}\textbf{89.98} & \cellcolor{blue!15}\textbf{99.86} \\ 
    \midrule[0.3pt]
    \textsc{GeoLDM}-respaced & 5   &  82.11 $\pm$ 0.1 & 27.96 $\pm$ 0.2 & 62.10 $\pm$ 0.1 & 59.20 $\pm$ 0.1 & 18.27 & 82.07\\
    \textbf{\textsc{\method}} & \textbf{5} & \cellcolor{blue!15}98.73 $\pm$ 0.0 & \cellcolor{blue!15}87.05 $\pm$ 0.4 & \cellcolor{blue!15}92.67 $\pm$ 0.3 & \cellcolor{blue!15}78.83 $\pm$ 0.2 & \cellcolor{blue!15}87.69 & \cellcolor{blue!15}99.69 \\
    \midrule[0.3pt]
    \textsc{GeoLDM}-respaced  & 4  & 72.35 $\pm$ 0.1 & 12.43 $\pm$ 0.3 & 45.24 $\pm$ 0.3 & 42.53 $\pm$ 0.3 & 8.30 & 84.65 \\
    \textbf{\textsc{\method}} & \textbf{4} & \cellcolor{blue!15}97.70 $\pm$ 0.1 & \cellcolor{blue!15}83.89 $\pm$ 0.3 & \cellcolor{blue!15}90.70 $\pm$ 0.2 & \cellcolor{blue!15}72.70 $\pm$ 0.3 & \cellcolor{blue!15}84.72 & \cellcolor{blue!15}99.72 \\
    \bottomrule[1.0pt]
\end{tabular}
\end{threeparttable}
}
\vspace{-10pt}
\end{table*}


\paragraph{Baselines} 
Our direct baseline is GeoLDM~\citep{xu2023geometriclatentdiffusionmodels}, the multi-step teacher of \method. To highlight the effectiveness of the distillation training, we compare directly to the GeoLDM with respaced timesteps (\Cref{subsec:reschedule}), denoted as GeoLDM-respaced. 
We further compare against competitive baselines spanning multiple generative paradigms, including: 1) Diffusion-based models, including EDM and its non-equivariant counterpart (GDM)~\citep{hoogeboom2022equivariantdiffusionmoleculegeneration}, as well as EDM-Bridge~\citep{wu2022diffusionbasedmoleculegenerationinformative} which extends EDM by incorporating informative prior bridges; 2) flow-based models, such as EquiFM~\citep{song2023equivariantflowmatchinghybrid}, GOAT~\citep{hong2025accelerating3dmoleculegeneration}, GeoBFN~\citep{song2024unifiedgenerativemodeling3d} and ADiT~\citep{joshi2025all}; 3) advanced few-step molecule generative methods, including AccGeoLDM~\citep{lacombe2024acceleratinggenerationmolecularconformations}, SLDM~\citep{ni2025straightlinediffusionmodelefficient}, GeoRCG~\citep{li2025geometricrepresentationconditionimproves} and MOLTD~\citep{zhang2025accelerating}. 
Notably, several methods that directly generate 2D bonds, including SemlaFlow~\citep{irwin2025semlaflowefficient3d}, are not compared against our models regarding 2D properties for fairness. This is because metrics such as stability and validity are essentially 2D metrics; they can easily reach very high values with 2D generative methods, yet this cannot reflect a method’s ability to learn 3D distributions. Therefore, we only include them to compare 3D Posebusters~\citep{buttenschoen2024posebusters} metrics in~\Cref{tab:posebusters}.

\paragraph{Results and Analysis}
\Cref{tab:main results} presents the main results, and we highlight several important observations as below.

First, with \textbf{the fewest NFEs}, \method achieves \textbf{the best stability and validity} on both QM9 and GEOM-DRUG. Specifically, on QM9, \method achieves 94.87\% molecule stability with only 8 NFEs, substantially outperforming GeoLDM-respaced, which achieves 70.78\%. It also surpasses strong few-step generative methods, including GeoRCG (89.08\% with 50 NFEs) and MOLTD (92.53\% with 12 NFEs). On GEOM-DRUG, \method even exceeds the dataset-level atom stability of 86.5\%, achieving 89.98\% atom stability. 

Second, under extremely smaller NFE budgets, such as 4 and 5 NFEs, \method significantly \textbf{lifts the near-unstable GeoLDM-respaced baseline} to a performance level \textbf{comparable to 1000-NFE baselines}. On QM9, without \method training, GeoLDM-respaced achieves only 12.43\% and 27.96\% molecule stability with 4 and 5 NFEs, indicating that almost no stable molecules can be generated with such few sampling steps. In contrast, 
\method achieves 83.89\% and 87.05\% molecule stability, approaching advanced performance with 250$\times$ and 200$\times$ less computation compared to the original 1000-step GeoLDM baseline.  The gain is even more pronounced on GEOM-DRUG, where the 5-step \method maintains the best performance among all compared methods.

Third, \method shows the \textbf{diversity metric V\&U} to a level \textbf{comparable with advanced methods}. For example, the 8-step \method achieves 87.51\% V\&U, compared with approximately 90\% for the baselines. It is known that reverse-KL-based objectives can lead to substantial diversity loss for the improvement of fidelity~\citep{yin2024onestepdiffusiondistributionmatching, jiang2026distributionmatchingdistillationmeets}, which is also validated in our setting as shown in~\Cref{fig:epoch}. This is because samples from the student distribution that lie outside the support of the teacher distribution incur a large penalty, encouraging $p_{\text{fake}}$ to concentrate on the high-probability modes of the teacher distribution while discarding low-probability regions. However, \method successfully mitigates this effect by introducing an $f$-divergence loss, whose effectiveness is further validated in~\Cref{subsec:ablation}. We note that although the V\&U score of \method remains slightly lower than those of some baselines, we view this \textbf{as an opportunity rather than a limitation} as discussed in DMDR~\citep{jiang2026distributionmatchingdistillationmeets}: steering the optimized distribution toward high-confidence regions is often desirable in practical few-step generation scenarios where sample quality is prioritized over exhaustive distribution coverage.

\begin{table}[t]
\centering
\caption{Conditional molecule generation results are evaluated on various properties of the molecule. We report the MAE(↓) between generated molecules. \textit{QM9} and \textit{Random} results serve as lower and upper bounds of MAE across all properties.}
\label{tab:conditional_results}
\begin{threeparttable}
\centering
\resizebox{0.8\linewidth}{!}{%
    \begin{tabular}{l | c c c c c c c}
    \toprule
    Property & NFE & $\alpha$& $\Delta \varepsilon$ & $\varepsilon_{\mathrm{HOMO}}$ & $\varepsilon_{\mathrm{LUMO}}$ & $\mu$ & $C_v$\\
    Units & - & Bohr$^3$ & meV & meV & meV & D & $\frac{\text{cal}}{\text{mol K}}$  \\
    \midrule
    QM9 & -  & 0.10 & 64 & 39 & 36 & 0.043 & 0.040  \\
    \midrule
    Random & - & 9.01  &  1470 & 645 & 1457  & 1.616  & 6.857   \\
    $N_\text{atoms}$ & - & 3.86  & 866 & 426 & 813 & 1.053  &  1.971 \\
    EDM      & 1000    & 2.76  & 655 & 356 & 584 & 1.111 & 1.101 \\
    GeoLDM & 1000 & 2.37 & 587 & 340 & 522 & 1.108 & 1.025 \\
    GeoBFN & \underline{100} &  2.34 & 577 & \textbf{328} & 516 & 0.998 & 0.949 \\
    GeoLDM-respaced & \textbf{32} & \underline{2.14} & \underline{565} & 342 & \underline{491} & \textbf{0.985} & \underline{0.933} \\
    \textbf{\textsc{\method}} & \textbf{32}  &  \textbf{2.12} & \textbf{556}  & \underline{330} & \textbf{489} & \underline{0.988} & \textbf{0.926}\\
    \bottomrule
\end{tabular}
}
\end{threeparttable}
\vspace{-7pt}
\end{table}

\subsection{Conditional Generation on Molecular Properties}

\label{subsection:conditional}

Conditional generation evaluates the model's ability to generate molecules that satisfy a specific property. Following GeoLDM~\citep{xu2023geometriclatentdiffusionmodels}, for each property condition, we inject it as a node feature in the few-step student, the score estimator $\mu_{\text{fake}}$ and the teacher model. The property values of the generated molecules are predicted by a pretrained property-specific classifier. Since the evaluation of this task relies solely on the 3D coordinates of the generated molecule, higher precision is required. As more NFEs yield finer-grained details, we choose a higher NFE budget (32 NFEs), which is still, to our knowledge, fewer than latest baselines in conditional generation on QM9. We provide implementation details and experiment results of models with fewer steps in~\Cref{appendix:conditional_config} and~\Cref{appendix:conditional}.

\begin{wrapfigure}[12]{R}{0.35\textwidth}
    \centering
    \includegraphics[width=\linewidth]{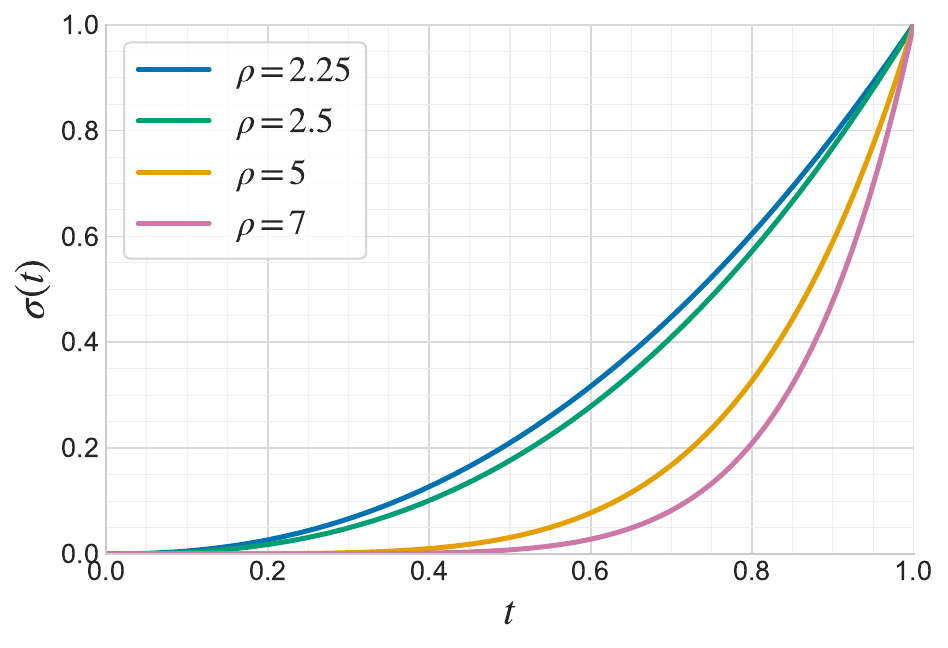}
    \caption{Noise schedules for different values of the exponent $\rho$ in the respaced timesteps in~\Cref{eq:reschedule}.}
    \label{fig:rho}
    \vspace{5pt}
\end{wrapfigure}

As presented in~\Cref{tab:conditional_results}, \method \textbf{achieves comparable or better results} than other methods in almost all metrics with \textbf{the fewest NFEs}. In $\varepsilon_{\mathrm{HOMO}}$ and $\varepsilon_{\mathrm{LUMO}}$, \method yields 330meV and 489meV MAE respectively compared to 340meV and 522meV produced by the 1000-step GeoLDM teacher. This demonstrates the applicability of \method to various scenarios.

\subsection{Ablation Studies}

\label{subsec:ablation}

\begin{table*}[t]
\centering
\caption{Ablation studies. \textbf{Left}: comparisons of using different $\rho$'s during sampling. \textbf{Middle}: ablation on the divergence regularizer. \textbf{Right}: ablation on the sampling method.}
\label{tab:ablation_summary}
\label{tab:rho}
\label{tab:divergence_ablation}
\label{tab:yesorno_ablation}
\vspace{4pt}
\begin{minipage}[b]{0.29\textwidth}
\centering
\textbf{Schedule exponent $\rho$}\\[2pt]
\scriptsize
\resizebox{0.8\linewidth}{!}{%
\begin{tabular}{ccc}
\toprule
$\rho$ & Mol Sta (\%) & V\&U (\%) \\
\midrule
2    & 69.93          & 81.27          \\
2.25 & \textbf{70.78} & \textbf{81.48} \\
2.75 & 68.46          & 81.35          \\
5    & 35.23          & 61.95          \\
\bottomrule
\end{tabular}
}
\end{minipage}
\hfill
\begin{minipage}[b]{0.34\textwidth}
\centering
\textbf{Divergence regularizer}\\[2pt]
\scriptsize
\resizebox{\linewidth}{!}{%
\begin{tabular}{lcc}
\toprule
Divergence Term & Mol\ Stab\ (\%) & V\&U (\%) \\
\midrule
DMD only                                   & 81.51          & 68.84 \\
Forward KL ($\lambda=0.1$)    & 58.09          & 59.10 \\
JS only                              & 83.34          & 72.36 \\
\textbf{DMD + JS} ($\lambda=0.1$)  & \textbf{83.89} & \textbf{72.70} \\
\bottomrule
\end{tabular}
}
\end{minipage}
\hfill
\begin{minipage}[b]{0.35\textwidth}
        \centering
\textbf{Sampling Method}\\[2pt]
        \label{tab:timestep}
        \scriptsize
        \resizebox{\linewidth}{!}{%
        \begin{tabular}{c c c c}
            \toprule[1.0pt]
            NFE & Respaced & DDIM & Consistency \\
            \midrule[0.8pt]
            4  & \textbf{12.43}{\tiny$\pm$0.30} & 0.01{\tiny$\pm$0.01}  & 0.00{\tiny$\pm$0.00}  \\
            8  & \textbf{70.78}{\tiny$\pm$0.39} & 5.06{\tiny$\pm$0.09}  & 6.71{\tiny$\pm$0.39}  \\
            16 & \textbf{94.61}{\tiny$\pm$0.18} & 41.84{\tiny$\pm$0.42} & 59.94{\tiny$\pm$0.49} \\
            \bottomrule[1.0pt]
        \end{tabular}
        }
\end{minipage}
\hfill
\vspace{-9pt}
\end{table*}

\paragraph{Effect of Respaced Schedule on Sampling Quality}

To isolate the contribution of the noise schedule, we sample with the GeoLDM teacher on QM9 without distillation under three          
conditions and report molecule stability: (1) uniform schedule with DDIM~\citep{song2020denoising}, as in GeoLDM~\citep{xu2023geometriclatentdiffusionmodels}; (2) uniform schedule with consistency sampling (Consistency); (3) \method's respaced schedule with consistency sampling (Respaced). As shown                   
in~\Cref{tab:timestep}, condition (3) \textbf{achieves 94.61\% molecule stability} and surpasses trained few-step models of       
MOLTD~\citep{zhang2025accelerating} at 16 NFEs \textbf{without any distillation training}. Consistency sampling alone yields a
small additional gain over DDIM.

\paragraph{Sensitivity to the Schedule Exponent $\rho$}

\label{subsec:rho}

To characterize the sensitivity of the sampling process to $\rho$ in~\Cref{eq:reschedule}, we vary $\rho$ and sample using the GeoLDM teacher on QM9 without distillation training; results are shown in~\Cref{tab:rho}. $\rho = 2.25$ yields the best stability and uniqueness. An excessively large $\rho$ concentrates sampling steps in the low-noise regime, losing the coarse structural deformation that
  early high-noise steps provide: sampling with $\rho = 5$ yields only 35.23\% molecule stability, compared to 70.78\% at $\rho =2.25$. Conversely, smaller $\rho$'s closely resemble the original DDIM schedule, limiting their ability to improve sample quality.

\paragraph{Effect of $f$-Divergence Choice on Stability and Diversity}

We compare the \method objective in~\Cref{eq:final_loss} with its variants: 1) DMD only ($\lambda = 0$); 2)                        
Forward-KL, substituting $f(r) := r\log r$ in~\Cref{eq:f-div}, which corresponds to the forward-KL objective $\text{KL}(p_{\text{real}}||p_{\text{fake}})$~\citep{xu2025onestepdiffusionmodelsfdivergence}; 3) Jensen-Shannon only ($\mathcal{L}_{\text{JS-only}} = \mathcal{L}_{\text{JS}}$). As shown in~\Cref{tab:divergence_ablation} and~\Cref{tab:DMD_only}, combining JS-divergence regularization with the DMD loss yields the best   
overall performance: JS-divergence \textbf{boosts Valid \& Unique by approximately 4\% at 4 NFEs} by increasing sampling diversity,  
  while DMD loss alone improves molecule stability but \textbf{reduces Valid \& Unique}, indicating collapsed probability support.
  Forward KL-divergence causes dramatic instability due to high-variance estimation~\citep{xu2025onestepdiffusionmodelsfdivergence}, incompatible with DMD's reverse-KL
  objective. More details can be found in~\Cref{appendix:f-div-ablation}.

\section{Conclusion}

We propose \method, a few-step generative model for 3D molecules that achieves up to $250\times$ acceleration over its 1000-step GeoLDM teacher while matching or exceeding teacher-level sample quality. Our approach adapts distribution matching distillation to the molecular domain through two key designs: an EDM-inspired respacing of generation timesteps that supplies the few-step student with a viable initialization, and a Jensen-Shannon divergence term that regularizes the reverse-KL DMD objective with a forward-KL component to mitigate mode-seeking behavior.  \method attains comparable generation quality in as few as 4 steps and state-of-the-art molecular stability and validity on QM9 and
GEOM-DRUG in 8 steps.
\bibliographystyle{plain}
\bibliography{references}

\appendix
\clearpage
\appendix

\counterwithin{figure}{section}
\counterwithin{table}{section}
\counterwithin{equation}{section}
\counterwithin{algorithm}{section}

\section*{Appendix}

\section{Additional Preliminaries}
\subsection{Molecule Diffusion Models}
\label[appendix]{appendix:molecule_diffusion}

We provide additional details on the molecule diffusion model summarized in~\Cref{section:preliminaries}. Following GeoLDM~\citep{xu2023geometriclatentdiffusionmodels}, a molecule with $N$ atoms is represented as $\mathcal{G}=\langle x,h\rangle$, where $x\in\mathbb{R}^{N\times 3}$ denotes 3D atom coordinates and $h\in\mathbb{R}^{N\times d}$ denotes node features. GeoLDM first maps the molecule into a continuous latent representation with an encoder $\mathcal{E}_\phi$:
\[
q_\phi(z_x,z_h\mid x,h)=\mathcal{N}(\mathcal{E}_\phi(x,h),\sigma_0^2 I).
\]
Here, $z_x$ is the coordinate latent and $z_h$ is the feature latent. The coordinate latent is constrained to the zero-center-of-mass subspace, i.e., $\sum_i z_{x,i}=0$, which removes global translation degrees of freedom. A decoder $\mathcal{D}_\varphi$ maps latent samples back to molecular coordinates and atom features.

The diffusion model operates entirely in this latent space. At each diffusion step, Gaussian noise is added to both latent coordinates and latent features according to a predefined schedule $(\alpha_t,\sigma_t)$:
\[
q(z_{x,t},z_{h,t}\mid z_{x,t-1},z_{h,t-1})
=
\mathcal{N}(z_{x,t};\sqrt{\alpha_t}z_{x,t-1},\sigma_t^2 I)
\cdot
\mathcal{N}(z_{h,t};\sqrt{\alpha_t}z_{h,t-1},\sigma_t^2 I).
\]
Equivalently, the noised latent at timestep $t$ can be written as
\[
z_t=\alpha_t z_0+\sigma_t\epsilon,\qquad \epsilon\sim\mathcal{N}(0,I),
\]
where $z_t=\langle z_{x,t},z_{h,t}\rangle$. During this process, the coordinate component is projected back to the zero-center-of-mass subspace so that the latent representation remains translation invariant.

The reverse denoising process is parameterized by a neural network that predicts the added noise. The original diffusion training objective can be written as
\[
\mathcal{L}_{DM}
=
\mathbb{E}_{(z_{x,0},z_{h,0})\sim q_\phi,\,\epsilon\sim\mathcal{N}(0,I),\,t}
\left[
\|\epsilon_x-\epsilon_{\theta,x}(z_{x,t},z_{h,t},t)\|^2
+
\|\epsilon_h-\epsilon_{\theta,h}(z_{x,t},z_{h,t},t)\|^2
\right].
\]
The denoiser therefore learns to recover both geometric and feature perturbations in latent space. In \method, the teacher model $\mu_{\text{real}}$, the few-step student generator $G_\theta$, and the fake score estimator $\mu_{\text{fake}}$ all inherit this latent-space parameterization.

For geometric consistency, GeoLDM uses equivariant networks such as EGNN~\citep{satorras2022enequivariantnormalizingflows} in the encoder, decoder, and denoising model. These networks preserve SO(3)-equivariance for coordinate channels and SO(3)-invariance for feature channels. Informally, rotating the coordinate input rotates the coordinate output in the same way, while the feature output remains unchanged. Together with the zero-center-of-mass constraint and an invariant Gaussian prior, this yields a latent molecular distribution that is invariant to global rigid transformations~\citep{xu2022geodiff}. This geometric structure is the basis for the invariance statement used in our distribution matching formulation.

At inference time, GeoLDM starts from Gaussian latent noise, applies iterative denoising in latent space, and decodes the final latent sample with $\mathcal{D}_\varphi$. \method preserves this inference interface but distills the denoising process into a small number of consistency-sampling steps, as described in~\Cref{subsec:dmd bkg,subsec:reschedule}.

\clearpage
\section{Formal Statements and Proofs}

\label[appendix]{appendix:A}

We provide the formal definition of invariance as follows:
\begin{definition}
    For a sample $z = \left< z_x, z_h \right> \in \mathbb{R}^{N \times (3+d)} \sim \mathcal{N}(0, I)$ in the data distribution and the model parameter $\theta$, a loss that is computed according to $\mathcal{L}(z;\theta)$ is \emph{invariant} to SO(3) \footnote{We only need to prove that the model is invariant to SO(3) only. This is due to the fact that we subtract the mean of all equivariant features with respect to the molecule $\mathcal{G}$ in each step of the EGNN backbone and therefore operate on a zero-center-of-mass space as defined in EDM~\citep{hoogeboom2022equivariantdiffusionmoleculegeneration}, which is invariant to  translation by definition.} group if and only if:
$$
\nabla_\theta\mathcal{L}(z;\theta) = \nabla_\theta\mathcal{L}(\mathbf{R}z;\theta)
$$
for any $\mathbf{R} \in SO(3)$.
\end{definition}

We use the notation $\mathbf{R}z$ to indicate group action of $\mathbf{R}$ on $z$, which is a matrix multiplication that gives $\mathbf{R}z = \left< z_x\mathbf{R}, z_h \right>$ where $\mathbf{R} \in SO(3)$.

\subsection{Proof of~\Cref{thm:invariance}}

\begin{proof}
$\forall \mathbf{R} \in SO(3), \mathbf{R}\mathbf{R}^T=I, \det(\mathbf{R})=1$, thus:
\begin{align*}
\nabla\mathcal{L}_\text{DMD}(\mathbf{R}z;t) &= - \int p(\epsilon) \big( s_{\text{real}}(\alpha_t G_{\theta}(\mathbf{R}z) + \sigma_t\epsilon, t) - s_{\text{fake}}(\alpha_t G_{\theta}(\mathbf{R}z) + \sigma_t\epsilon, t) \big) \frac{dG_\theta(\mathbf{R}z)}{d\theta} \, d\epsilon \\
\intertext{Given the equivariant property of EGNN \citep{satorras2022enequivariantnormalizingflows}, we denote $z' = \langle z_x', z_h' \rangle = G_\theta(z)$, then $G_\theta(\mathbf{R}z) = \langle z_x'\mathbf{R}, z_h' \rangle$. Therefore, we denote $\bar{\mathbf{R}} = \operatorname{diag}\{\mathbf{R}, I_d\}$, then $G_\theta(\mathbf{R}z) = G_\theta(z) \bar{\mathbf{R}}$, $\bar{\mathbf{R}} \bar{\mathbf{R}}^T = I$ and $\det(\bar{\mathbf{R}}) = 1$.}
\nabla\mathcal{L}_\text{DMD}(\mathbf{R}z;t) &= -\int p(\epsilon) \big( s_{\text{real}}(\alpha_t G_{\theta}(z) + \sigma_t\epsilon\bar{\mathbf{R}}^T, t) - s_{\text{fake}}(\alpha_t G_{\theta}(z) + \sigma_t\epsilon\bar{\mathbf{R}}^T, t) \big)\bar{\mathbf{R}}\bar{\mathbf{R}}^T\frac{dG_\theta(z)}{d\theta} \, d\epsilon\\
&= -\int p(\epsilon) \big( s_{\text{real}}(\alpha_t G_{\theta}(z) + \sigma_t\epsilon\bar{\mathbf{R}}^T, t) - s_{\text{fake}}(\alpha_t G_{\theta}(z) + \sigma_t\epsilon\bar{\mathbf{R}}^T, t) \big)\frac{dG_\theta(z)}{d\theta} \, d\epsilon\\
&= -\int p(\epsilon\bar{\mathbf{R}}^T) \big( s_{\text{real}}(\alpha_t G_{\theta}(z) + \sigma_t\epsilon\bar{\mathbf{R}}^T, t) - s_{\text{fake}} (\alpha_t G_{\theta}(z) + \sigma_t\epsilon\bar{\mathbf{R}}^T, t) \big)\frac{dG_\theta(z)}{d\theta} \, d\epsilon\\
&= - \int p(\epsilon) \big( s_{\text{real}}(\alpha_t G_{\theta}(z) + \sigma_t\epsilon, t) - s_{\text{fake}}(\alpha_t G_{\theta}(z) + \sigma_t\epsilon, t) \big) \frac{dG_\theta(z)}{d\theta} \, d\epsilon \\
&= \nabla\mathcal{L}_\text{DMD}(z;t)
\end{align*}
\end{proof}

It is worth noting that the $\epsilon$ added during the Markovian forward process is \textit{not} equivariant, we need to prove that in expectation sense this does not matter. Given that $det({\mathbf{R}}) = 1, \forall \mathbf{R} \in SO(3),$ 
\begin{align*}
\nabla\mathcal{L}_\text{DMD}(z;t) &=
        - \int p(\epsilon) \big( s_{\text{real}}(\alpha_t G_{\theta}(z) + \sigma_t\epsilon, t) - s_{\text{fake}}(\alpha_t G_{\theta}(z) + \sigma_t\epsilon, t) \big) \frac{dG_\theta(z)}{d\theta} \, d\epsilon \\
&=  - \int p(\epsilon') \big( s_{\text{real}}(\alpha_t G_{\theta}(z) + \sigma_t\epsilon' \bar{\mathbf{R}}^{-1}, t) - s_{\text{fake}}(\alpha_t G_{\theta}(z) + \sigma_t\epsilon' \bar{\mathbf{R}}^{-1}, t) \big) \frac{dG_\theta(z)}{d\theta} \, d\epsilon'
\end{align*}
where $\epsilon' = \epsilon \bar{\mathbf{R}}$.

\clearpage
\section{Experimental Settings}
\label[appendix]{appendix:B}

\begin{algorithm}[t]
\caption{Training procedure}
\label{alg:training}
\begin{algorithmic}[1]
\Require Frozen teacher $\mu_{\text{real}}$, generator $G$, auxiliary score model $\mu_{\text{fake}}$, discriminator $D$, EMA generator $G_{\text{EMA}}$
\State Initialize $G$ with respaced schedule
\For{each epoch $e$}
    \For{each molecular minibatch $x$}
        \State Preprocess $x$; construct masks and optional context $c$
        \State $x_e \gets \mathrm{VAE.encode}(x)$
        \State Sample step number $k \in \mathcal{K}$
        \State $t \gets$ sample diffusion timestep with cutoff $T_{\max}$
        \State $t \gets \mathrm{EDM\_schedule}(t, k, s{=}10^{-4}, \rho{=}2.25)$
        \State $z_0 \gets G(c, k)$
        \Statex
        \Statex \textit{// Update auxiliary score model and discriminator}
        \For{$j = 1$ \textbf{to} $N_{\text{ratio}}$}
            \State $z_t^{f} \gets \mathrm{corrupt}(\mathrm{sg}(z_0), t)$
            \State $z_t^{r} \gets \mathrm{corrupt}(\mathrm{sg}(x_e), t)$
            \State Evaluate $\mu_{\text{fake}}$ on $z_t^{f}$ and $z_t^{r}$
            \State Update $D$ using real and fake intermediate features from $\mu_{\text{fake}}$
            \If{$e$ is after warm-up}
                \State Update $\mu_{\text{fake}}$ with the fake-sample diffusion objective
            \EndIf
        \EndFor
        \Statex
        \Statex \textit{// Update generator}
        \State $z_t \gets \mathrm{corrupt}(z_0, t)$
        \State $s_{\text{real}} \gets \mu_{\text{real}}(z_t, t)$, \quad $s_{\text{fake}} \gets \mu_{\text{fake}}(z_t, t)$
        \State Estimate divergence weight from the output of $D$
        \State Update $G$ with the reweighted DMD objective
        \State $G_{\text{EMA}} \gets \mathrm{EMA}(G_{\text{EMA}}, G)$
    \EndFor
\EndFor
\end{algorithmic}
\end{algorithm}

\subsection{Discriminator Design}
  \label[appendix]{appendix:gan_design}

  The discriminator $D_\psi$ takes as input three intermediate layers of $\mu_{\text{fake}}$: the second, fifth, and seventh layers for QM9, and the zeroth,
  first, and second layers for GEOM-DRUG. We denote the selected representations of a molecule $x \in \mathbb{R}^{N \times D}$ by \texttt{mu\_fake\_out\_0}, \texttt{mu\_fake\_out\_1}, and
  \texttt{mu\_fake\_out\_2}, respectively.

  For each extracted representation $\texttt{mu\_fake\_out\_i} \in \mathbb{R}^{N \times D}$, where $i \in \{0,1,2\}$,
  $N$ is the number of atoms, and $D$ is the hidden dimension, we introduce a learnable query token $\texttt{q\_i} \in \mathbb{R}^{D}$. This query
  token aggregates information from the corresponding layer through cross-attention~\citep{vaswani2017attention}:
  \[
  \texttt{output\_i}
  =
  \operatorname{CrossAttn}
  \left(
  \text{query}=\texttt{q\_i},
  \text{key}=\texttt{mu\_fake\_out\_i},
  \text{value}=\texttt{mu\_fake\_out\_i}
  \right).
  \]
  The resulting representation is then passed through a layer-specific MLP to obtain the discriminator logit:
  \[
  \texttt{logits\_i} = \operatorname{MLP}_i(\texttt{output\_i}).
  \]
  Each logit is converted into a probability score $\texttt{p\_i}$ using the softmax activation. The final discriminator score $D_\psi(x)$ is computed as the
  average of the three probability scores:
  \[
  \texttt{p}
  =
  \frac{1}{3}
  \sum_{i=0}^{2} \texttt{p\_i}.
  \]

To achieve faster convergence and increase training stability, we also adopt an R1 regularization loss~\citep{mescheder2018trainingmethodsgansactually}:
\begin{align*}
    R_1(\psi) = \mathbb{E}_{x\sim p_{\text{data}}} \big[ 
    ||\nabla D_\psi(x)||^2
    \big].
\end{align*}

\subsection{Implementation details for unconditional generation}

\label[appendix]{appendix:config}

\paragraph{QM9} We implement our distillation framework on top of GeoLDM~\citep{xu2023geometriclatentdiffusionmodels} within the PyTorch framework~\citep{Paszke2017AutomaticDI}, using the EGNN-based equivariant score network~\citep{satorras2022enequivariantnormalizingflows} as both the teacher and the student generator. Following the GeoLDM implementation, we operate in the learned continuous latent space with invariant feature dimensionality of 2, and we retain the teacher's architectural hyperparameters for the student generator $G_\theta$, the fake score network $\mu_{\text{fake}}$, and the discriminator backbone. The student is initialized from the pretrained teacher's weights, while $\mu_{\text{fake}}$ is initialized as a copy of the teacher and subsequently updated online during distillation.

We adopt the Adam optimizer~\citep{kingma2017adammethodstochasticoptimization} for all three networks. We observe that 4-step generation training typically suffers from under-fitting and requires a larger learning rate to achieve a better minimum. For the $5$-step and $8$-step regimes, we use $G_{lr} = 8 \times 10^{-7}$, $\mu_{\text{fake},lr} = 32 \times 10^{-7}$, and $\text{disc}_{lr} = 16 \times 10^{-5}$. For the $4$-step regime, we scale all three learning rates by a factor of $10$, yielding $G_{lr} = 8 \times 10^{-6}$, $\mu_{\text{fake},lr} = 32 \times 10^{-6}$, and $\text{disc}_{lr} = 16 \times 10^{-4}$. The relative ratio between the three learning rates ($1 : 4 : 200$) is preserved across both regimes, which we found important for maintaining stable adversarial dynamics between the generator, the fake score estimator, and the discriminator. An exponential moving average of the student weights is maintained with decay $0.9999$ and used for all evaluation.

We provide more details on the self-rollout procedure. We also note that using 1, 1, and 3 respectively for generation of 4 steps, 5 steps, and 8 steps as the lower bound for the number of generated samples passed to $\mu_{\text{fake}}$ is optimal. During each training iteration, the number of generation steps $K$ is sampled from one of three curricula: $K \in \{1, 2, 3, 4\}$, $K \in \{1, 2, 3, 4, 5\}$, or $K \in \{3, 4, 5, 6, 7, 8\}$. The first two curricula target few-step inference, while the third targets higher-fidelity multi-step generation. Sampling $K$ uniformly within a curriculum exposes the student to a range of rollout horizons and encourages consistent generation quality across the supported inference budgets. At each step of the rollout, we apply the distribution matching gradient following DMD2~\citep{yin2024improveddistributionmatchingdistillation}, with the fake score network updated via a standard denoising objective on the student's generated samples.

The total training loss combines the distribution matching objective with two auxiliary terms: an adversarial GAN loss and a Jensen-Shannon divergence loss computed between the teacher and student score distributions. Across all experiments, we fix the GAN loss coefficient for $\mu_{\text{fake}}$ to $\lambda_{\text{GAN}} = 0.2$ and the Jensen-Shannon loss coefficient to $\lambda_{\text{JS}} = 0.1$. We observed that both auxiliary losses improve sample quality and stability but must be kept subordinate to the primary distribution matching loss to avoid mode collapse. Training is performed on RTX 4090 GPUs. For QM9, we use a batch size of 32 and the student converges within 5 epochs (15,000 iterations). 

\paragraph{GEOM-DRUG} We implement our distillation framework on top of GeoLDM~\citep{xu2023geometriclatentdiffusionmodels} in PyTorch~\citep{Paszke2017AutomaticDI}. Following
  GeoLDM, both the teacher and student models use an EGNN-based equivariant score network~\citep{satorras2022enequivariantnormalizingflows}. For GEOM-
  DRUG, we perform distillation in the learned continuous latent space with latent feature dimensionality $2$. Unless otherwise stated, the student generator $G_\theta$, the
  fake score network $\mu_{\mathrm{fake}}$, and the discriminator backbone follow the same architectural hyperparameters as the pretrained teacher. In
  particular, we use hidden dimension $256$, four EGNN layers, normalization factors $[1,4,10]$, and a global normalization factor of $1$. An exponential moving average of the student weights is maintained with decay $0.9999$ and used for all evaluation.

  The student is initialized from the pretrained teacher checkpoint, while $\mu_{\mathrm{fake}}$ is initialized as a copy of the teacher and updated online
  during distillation. We use the polynomial noise schedule \texttt{polynomial\_2} with noise precision $10^{-5}$ and train with the $\ell_2$ diffusion loss.
  The minimum diffusion time is set to $T_{\min}=0.02$, with the same value used during pretraining. For the GEOM-DRUG experiment reported here, we target
  $5$-step generation and set the rollout step number to $K=5$.

  We optimize the student generator, fake score network, and discriminator using Adam~\citep{kingma2017adammethodstochasticoptimization}. The learning rates are
  set to
  \[
  G_{\mathrm{lr}} = 5 \times 10^{-6}, \qquad
  \mu_{\mathrm{fake}, \mathrm{lr}} = 4 \times 10^{-6}, \qquad
  D_{\mathrm{lr}} = 2 \times 10^{-4}.
  \]
  The R1 regularization weight is set to $0.001$, with perturbation
  scale $\sigma=0.01$.

  The total objective combines the distribution matching loss with auxiliary adversarial and divergence terms. For this GEOM-DRUG setting, the adversarial network coefficient, which is related to both $\mu_{\text{fake}}$ and $D_\psi$, is set to $0.47$. We additionally enable the Jensen-Shannon divergence objective
  with coefficient $0.03$. The consistency regularization coefficient and additional regularization coefficient are both set to $0$.

  Training is performed with distributed data parallelism on $8$ A100 GPUs. We use a per-device batch size of $6$, corresponding to an effective batch size of $48$ and the student converges within 20,000 iterations, with 3,000 iterations of discriminator warmup.

\subsection{Quality-stability trade-off}

\label{appendix:q&s tradeoff}

The experiments are conducted using the same parameters as used to train the models in~\Cref{appendix:config}. However, we simply elongate the training time. Due to the higher magnitude of the learning rate for the generation of 4 steps, longer training leads to overshooting, and we do not report them in ablation on training time.

\subsection{Conditional Generation}

\label[appendix]{appendix:conditional_config}

Following previous conventions \citep{xu2023geometriclatentdiffusionmodels}, we evaluate conditional generation results on 6 properties. The condition is added as an additional node feature in the molecule representation. For each molecule property, we train a GeoLDM teacher for 3000 epochs conditioning on it and then distill a \method student model accordingly. We judge the results by training a classifier separately that outputs the value of this property given a molecule. Since generally speaking conditional generation is more difficult, we present the results of our 32 step models.

We train \method on the QM9 dataset, conditioning on the $\alpha, \Delta \epsilon$, HOMO, LUMO, $\mu$, and Cv respectively following~\citep{xu2023geometriclatentdiffusionmodels} and~\citep{hoogeboom2022equivariantdiffusionmoleculegeneration}. The student distills a pre-trained conditional GeoLDM teacher (\texttt{exp\_cond\_alpha} for instance) into a 32-step generator using the f-divergence variant of DMD~\citep{xu2025onestepdiffusionmodelsfdivergence} with Jensen--Shannon divergence as the final objective as specified in~\Cref{sec:methods}.

\subsubsection{Architecture}
\label{app:impl:arch}

The student inherits the teacher's architecture and is initialized from the teacher weights. Both networks operate on continuous point-structured geometric latents $z = \langle z_x, z_h \rangle$ with $k = 1$ invariant channel per node, parameterized by an EGNN denoiser with 9 layers and hidden width 192. Feature normalization factors are $[1, 8, 1]$ for $(x, h_{\mathrm{cat}}, h_{\mathrm{int}})$, and atomic charges are excluded from the input. An exponential moving average of the student weights is maintained with decay $0.9999$ and used for all evaluation.

\begin{table}[h]
\centering
\caption{Architecture hyperparameters.}
\begin{tabular}{lc}
\toprule
Hyperparameter & Value \\
\midrule
Hidden features (\texttt{nf}) & 192 \\
Number of layers (\texttt{n\_layers}) & 9 \\
Latent features (\texttt{latent\_nf}) & 1 \\
Normalization factors & $[1, 8, 1]$ \\
EMA decay & 0.9999 \\
\bottomrule
\end{tabular}
\label{tab:impl:arch}
\end{table}

\subsubsection{Diffusion Settings}
\label{app:impl:diffusion}

The teacher uses a polynomial-2 noise schedule with precision $10^{-5}$ and an L2 $\epsilon$-prediction objective. $[T_{\min}, T]$ with $T_{\min} = 0.02$ to avoid the unstable low-noise tail is applied to the timestep sampler used for the fake-score network's denoising-score-matching update.

\begin{table}[h]
\centering
\caption{Diffusion settings.}
\begin{tabular}{lc}
\toprule
Parameter & Value \\
\midrule
Noise schedule & polynomial-2 \\
Noise precision & $10^{-5}$ \\
Diffusion loss & L2 \\
Distillation steps $K$ & 32 \\
Minimum timestep for score estimation $T_{\min}$ & 0.02 \\
\bottomrule
\end{tabular}
\label{tab:impl:diffusion}
\end{table}

The discriminator is regularized with an R1 gradient penalty~\citep{mescheder2018trainingmethodsgansactually} of weight $10^{-3}$ at noise scale $\sigma = 0.01$.

\begin{table}[h]
\centering
\caption{Loss coefficients and regularization.}
\begin{tabular}{lc}
\toprule
Parameter & Value \\
\midrule
f-divergence type & Jensen--Shannon (\texttt{use\_js}) \\
f-divergence coefficient & 0.1 \\
GAN coefficient ($\mu_{\text{fake}}$ and discriminator, $\lambda_f$) & 0.2 \\
R1 gradient penalty weight & $10^{-3}$ \\
R1 noise scale $\sigma$ & 0.01 \\
\bottomrule
\end{tabular}
\label{tab:impl:loss}
\end{table}

\begin{table}[H]
\caption{Learning rates.}
\centering
\begin{tabular}{lc}
\toprule
Component & Learning rate \\
\midrule
Generator $G_\theta$ & $8 \times 10^{-7}$ \\
Fake-score network $\mu_{\mathrm{fake}}$ & $3.2 \times 10^{-6}$ \\
Discriminator $D$ & $1.6 \times 10^{-4}$ \\
\bottomrule
\end{tabular}
\label{tab:impl:lr}
\end{table}

\begin{table}[H]
\centering
\caption{Training and evaluation configuration.}
\begin{tabular}{lc}
\toprule
Parameter & Value \\
\midrule
Dataset & QM9 (second half, 50K) \\
Include charges & No \\
Batch size & 32 \\
Stability samples & 10{,}000 \\
GPU & NVIDIA RTX 4090 \\
\bottomrule
\end{tabular}
\label{tab:impl:train}
\end{table}

\subsubsection{Optimization}
\label{app:impl:opt}

Following DMD2~\citep{yin2024improveddistributionmatchingdistillation} and f-distill~\citep{xu2025onestepdiffusionmodelsfdivergence}, the three networks are optimized at distinct learning rates with Adam: the generator $G_\theta$ at $8 \times 10^{-7}$, the fake-score network $\mu_{\mathrm{fake}}$ at $3.2 \times 10^{-6}$, and the discriminator at $1.6 \times 10^{-4}$. The auxiliary networks ($\mu_{\mathrm{fake}}$ and $D$) are updated $5$ times per generator update (\texttt{step\_ratio} $= 5$), to better track the moving student distribution before each f-distill gradient is computed. The lower bound of the self-rollout step number is set to 3 for 8, 12 and 32 NFEs while 1 for 5 NFEs.

\subsubsection{Training and Evaluation}
\label{app:impl:train}

Training is performed on the QM9 second-half split (50K molecules) with batch size $32$ on a single NVIDIA RTX 4090 GPU. Evaluation is performed by running three runs,  drawing $10{,}000$ conditional samples each, and we report the conditional MAE against a pretrained EGNN property predictor $\omega$ trained on the first 50K split, following the protocol of EDM~\citep{hoogeboom2022equivariantdiffusionmoleculegeneration} and GeoLDM~\citep{xu2023geometriclatentdiffusionmodels}.

\clearpage
\section{Results}
\label[appendix]{appendix:C}
\paragraph{Sampling protocol} Following previous conventions, all outcomes below on both QM9 and DRUG dataset are obtained by running the test for 3 times. We sample 10000 molecules for each test. We report the mean for each metric. Our models are trained until Valid \& Unique drops.

\subsection{Unconditional Molecule Generation Results}

\paragraph{PoseBusters Evaluation} PoseBusters~\citep{buttenschoen2024posebusters} evaluate 3D conformation properties of the generated molecules. They include the following metrics: 1. Atoms connected (which checks if all generated atoms form a single connected molecule); 2. Bond angles (which checks if every bond angle is within tolerance of the range expected from distance geometry); 3. Bond lengths (which checks if every bond length is within tolerance of its reference value); 4. Internal steric clash (which checks if any pair of non-bonded atoms is closer than allowed); 5. Aromatic ring flatness (which checks if all atoms in each aromatic ring lie in one plane); 6. Double bond flatness (which checks if the atoms around each carbon–carbon double bond lie in one plane); and 7. Internal energy (which checks if the conformer's UFF~\citep{rappe1992uff} energy is not unreasonably higher than that of a force-field-relaxed version of the same molecule). Notably, validity means the total percentage of molecules that pass all these checks and are InChI~\citep{heller2015inchi} convertable, which is different from the validity metrics used in GeoLDM~\citep{xu2023geometriclatentdiffusionmodels}.

As presented in~\Cref{tab:posebusters}, FlashMol surpasses the teacher and untrained model in all stability metrics on QM9 and GEOM-DRUGS. Besides, FlashMol also achieves better generation quality in terms of bond lengths, ring flatness, double bond flatness and no-steric-clash percentage compared to SemlaFlow~\citep{irwin2025semlaflowefficient3d} and ADiT~\citep{joshi2025all} with far fewer NFEs (only 8).

\begin{table}[t]
    \label{posebuster}
  \centering
  \footnotesize
  \setlength{\tabcolsep}{2pt}
    \caption{PoseBusters evaluation on QM9 (a) and GEOM-DRUGS (b). On QM9, we compare against the 1000-NFE GeoLDM teacher under 8 NFEs. On GEOM-DRUGS, we underline the best entry in each row (the fewest NFEs for NFE). We \colorbox{blue!15}{highlight} metrics in which our student model \method surpasses GeoLDM-respaced under the same NFEs.}
        \vspace{5pt}
  \begin{minipage}[t]{0.455\linewidth}
    \centering
    {\normalsize (a) QM9}\par\vspace{4pt}
    \resizebox{\linewidth}{!}{
    \begin{tabular}{lccc}
      \toprule
      Metrics (\% pass) & \shortstack{GeoLDM-\\respaced} & \method & \shortstack{GeoLDM\\teacher} \\
      \midrule
      Atoms connected    & 98.65  & \cellcolor{blue!15}99.81  & 99.89  \\
      Bond angles        & 99.80  & \cellcolor{blue!15}100.00 & 100.00 \\
      Bond lengths       & 99.64  & \cellcolor{blue!15}99.99  & 99.98  \\
      Ring flat          & 100.00 & 100.00                    & 100.00 \\
      Double bond flat   & 97.02  & \cellcolor{blue!15}98.53  & 99.30  \\
      Internal energy    & 92.90  & \cellcolor{blue!15}96.66  & 96.45  \\
      No steric clash    & 99.60  & \cellcolor{blue!15}99.95  & 99.87  \\
      Validity           & 78.29  & \cellcolor{blue!15}94.40  & 90.20  \\
      Uniqueness         & 96.10  & 94.53                     & 98.90  \\
      \midrule
      NFE                & 8      & 8                         & 1000   \\
      \bottomrule
    \end{tabular}}
  \end{minipage}\hfill
  \begin{minipage}[t]{0.52\linewidth}
    \centering
    {\normalsize (b) GEOM-DRUGS}\par\vspace{4pt}
    \resizebox{\linewidth}{!}{
    \begin{tabular}{lcccc}
      \toprule
      Metric & \method & \shortstack{GeoLDM-\\respaced} & SemlaFlow & ADiT \\
      \midrule
      Validity             & \cellcolor{blue!15}\underline{99.66}  & 96.42              & 93.90              & 95.30              \\
      Uniqueness           & 85.84                                 & 96.34              & \underline{100.00} & \underline{100.00} \\
      Novelty              & 85.43                                 & \underline{96.25}  & $-$                & $-$                \\
      Atoms connected      & \underline{100.00}                    & \underline{100.00} & 92.30              & 93.00              \\
      Bond angles          & \cellcolor{blue!15}91.20              & 17.69              & \underline{94.80}  & 92.30              \\
      Bond lengths         & \cellcolor{blue!15}\underline{99.20}  & 75.18              & 94.60              & 92.50              \\
      Ring flat            & \cellcolor{blue!15}\underline{100.00} & 99.98              & 94.90              & 95.40              \\
      Double bond flat     & \underline{100.00}                    & \underline{100.00} & 94.20              & 95.30              \\
      Internal energy      & \cellcolor{blue!15}88.30              & 15.41              & \underline{94.80}  & 91.30              \\
      No steric clash      & \cellcolor{blue!15}\underline{97.61}  & 63.38              & 92.00              & 91.80              \\
      PoseBusters valid    & \cellcolor{blue!15}85.31              & 9.05               & \underline{87.50}  & 85.30              \\
      \midrule
      NFE                  & \underline{8}                         & \underline{8}      & 100                & 500/1000           \\
      \bottomrule
    \end{tabular}}
  \end{minipage}
  \vspace{6pt}
  \label{tab:posebusters}
  \vspace{-7pt}
\end{table}

\paragraph{Stability and Diversity Trade-Offs} We train the 8-NFE and 5-NFE model using the configurations in~\Cref{appendix:B}. During training, although molecule stability generated by the models keeps increasing, Valid \& Unique metrics keeps declining by a small margin as shown in~\Cref{fig:epoch}.

\begin{figure}[h]
    \centering
    \includegraphics[width=0.48\textwidth]{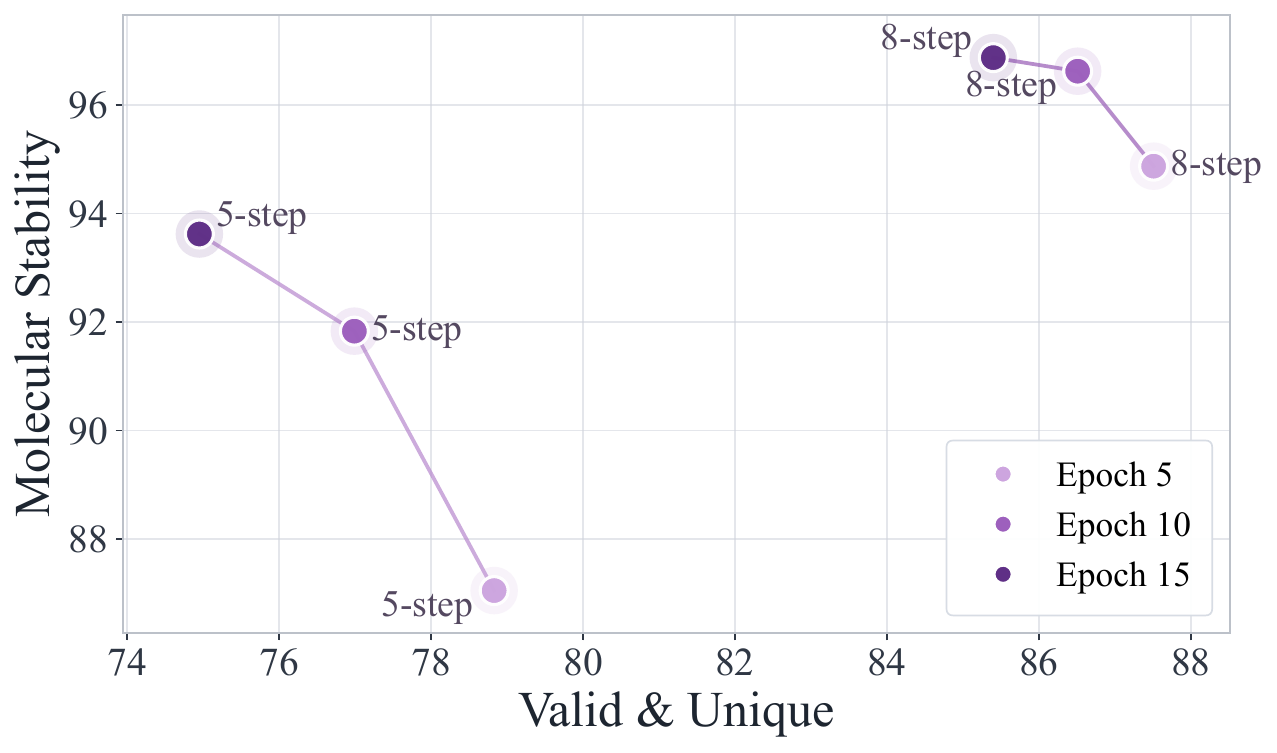}
    \caption{Training dynamics of the distilled student across epochs. Darker points correspond to later training epochs. While atom and molecule stability improve with training, uniqueness decreases monotonically
    }
    \label{fig:epoch}
\end{figure}

\subsection{Conditional Molecule Generation Results}

\label[appendix]{appendix:conditional}
Here, we present the test results of conditional generation trained with fewer NFEs. All hyperparameters are the same for the 12-step and 8-step models. We train each model until convergence: 5 epochs for the 12-NFE model, 6 epochs for the 8-NFE model, and 10 epochs for the 5-NFE model.

\begin{table}[h]  
\centering
\vspace{-10pt}
\caption{Conditional generation results are evaluated on various properties of the molecule. We report the MAE(↓) between generated molecules. The property of each generated molecule is predicted by a pretrained EGNN classifier. \textit{QM9} and \textit{Random} results can be viewed as lower and upper bounds of MAE on all properties. We \colorbox{blue!15}{highlight} metrics in which our student model surpasses the respaced teacher using the same NFEs.}
\begin{threeparttable}
    \begin{tabular}{l | c c c c c c c}
    \toprule
    Property & NFE & $\alpha$& $\Delta \varepsilon$ & $\varepsilon_{\mathrm{HOMO}}$ & $\varepsilon_{\mathrm{LUMO}}$ & $\mu$ & $C_v$\\
    Units & - & Bohr$^3$ & meV & meV & meV & D & $\frac{\text{cal}}{\text{mol K}}$  \\
    \midrule
    QM9 & -  & 0.10 & 64 & 39 & 36 & 0.043 & 0.040  \\
    \midrule
    Random & - & 9.01  &  1470 & 645 & 1457  & 1.616  & 6.857   \\
    $N_\text{atoms}$ & - & 3.86  & 866 & 426 & 813 & 1.053  &  1.971 \\
    EDM      & 1000    & 2.76  & 655 & 356 & 584 & 1.111 & 1.101 \\
    GeoLDM & 1000 & 2.37 & 587 & 340 & 522 & 1.108 & 1.025 \\
    GeoBFN & 100 &  2.34 & 577 & 328 & 516 & 0.998 & 0.949 \\
        \midrule[0.3pt]
    GeoLDM-respaced & 12 & 2.70 & 670 & 383 & 591 & 1.058 & 1.087 \\
    \textsc{\method} & 12 & 2.70 & \cellcolor{blue!15}646  & \cellcolor{blue!15}367 & \cellcolor{blue!15}578 & 1.082 & 1.117 \\
        \midrule[0.3pt]
    GeoLDM-respaced & 8 & 3.20 & 750 & 423 & 686 & 1.111 & 1.301 \\
    \textsc{\method} & 8 & \cellcolor{blue!15}2.98 & \cellcolor{blue!15}714 & \cellcolor{blue!15}399 & \cellcolor{blue!15}611 & 1.111 & \cellcolor{blue!15}1.230 \\
        \midrule[0.3pt]
    GeoLDM-respaced & \textbf{5} & 5.24 & 935 & 591 & 910 & 1.343 & 1.706 \\
    \textsc{\method} & \textbf{5} & \cellcolor{blue!15}3.88 & \cellcolor{blue!15}814 & \cellcolor{blue!15}538 & \cellcolor{blue!15}747 & 1.830 & \cellcolor{blue!15}1.428 \\
    
    \bottomrule
\end{tabular}
\end{threeparttable}
\vspace{-5pt}
\end{table}

\subsection{Ablation Study Results}

\paragraph{Effect of $f$-Divergence Choice on Stability and Diversity}

\label[appendix]{appendix:f-div-ablation}

We use the same configuration as elaborated in~\Cref{appendix:config} and compare our full objective (DMD + JS) against a Jensen-Shannon term only, reverse KL + forward KL, or a DMD only variant. The corresponding loss objectives are: $\mathcal{L}_{\text{DMD}}$, $\mathcal{L}_{\text{DMD}} + \lambda \cdot \mathcal{L}_{\text{forward}}$ and $\mathcal{L}_{\text{JS}}$ respectively, where $\lambda=0.1$ and $\mathcal{L}_{\text{forward}}$ is computed with $f :=r\log r$ in~\Cref{eq:f-div}. We train all models until V\&U metric drops (which is 5 epochs for all models). Higher is better for all metrics. Results are reported as means over 3 sampling runs of 10000 molecules each on QM9 dataset. The results are shown in~\Cref{tab:divergence_ablation}.

We also report the results trained without the Jensen-Shannon objective across 4, 5 and 8 steps in~\Cref{tab:DMD_only}. All models are distilled under the configurations as specified in~\Cref{appendix:config} from the same GeoLDM teacher with $K=4,5,8$ generation steps respectively and evaluated at epoch 5 where all training converge. Results trained with regularizers can be found in~\Cref{tab:main results}. Higher is better for all metrics. Results are reported as means over three sampling runs of 10000 molecules each. The results are shown in~\Cref{tab:yesorno_ablation}.

\begin{table}[h]
\centering
\caption{DMD-only variant}
\label{tab:DMD_only}
\begin{tabular}{lcc}
\toprule
NFE & Mol Stab (\%) & V\&U (\%) \\
\midrule
4 & 81.51 & 68.84 \\
5 & 88.60 & 78.19 \\
8 & 95.32 & 86.89 \\
\bottomrule
\end{tabular}
\end{table}

\subsection{Transferability of FlashMol to Other Backbones}

\label[appendix]{app:semla}

To demonstrate that FlashMol transfers accross backbones, we choose another strong molecule generative model, SemlaFlow~\cite{irwin2025semlaflowefficient3d}, as the baseline. DMD~\cite{yin2024onestepdiffusiondistributionmatching} is designed for continuous space generation. However, SemlaFlow involves both generation in continuous space (for 3D coordinates) and in discrete space (for atom types, bond orders and formal charges). Therefore, we integrate Discrete Flow Maps~\cite{potaptchik2026discrete} to provide supervision signal for the discrete part of the output, whereas we use our $f$-divergence objective for the continuous part. We present the adaptation of Discrete Flow Maps and our results for SemlaFlow below.

\paragraph{Adapatation of Discrete Flow Maps} DMD is well defined only for the Gaussian coordinate path. Atom types, bond orders, and formal charges in SemlaFlow~\cite{irwin2025semlaflowefficient3d} instead follow a discrete corruption process $\mathbf{x}_s \sim \mathrm{Cat}\big(s\,\mathbf{p}_1 + (1-s)\,\mathcal{U}\big)$, whose state space admits no score, so the DMD gradient cannot provide signal to the categorical heads of the student. To address this issue, we distill the discrete heads with the \textbf{progressive self-distillation (PSD) loss of Discrete Flow Maps}~\cite{potaptchik2026discrete}, which operates directly on the probability simplex rather than forcing one-hot states into continuous space. PSD parameterizes the few-step map by the mean denoiser $\psi_{s,t}(\mathbf{x}) \in \Delta^{K-1}$, the time-averaged posterior over the clean token, and enforces the semigroup identity
\begin{equation}
\psi_{s,t}(\mathbf{x}) = \alpha_{s,u,t}\,\psi_{s,u}(\mathbf{x}) + \beta_{s,u,t}\,\psi_{u,t}\big(X_{s,u}(\mathbf{x})\big),
\qquad
X_{s,u}(\mathbf{x}) = \tfrac{1-u}{1-s}\,\mathbf{x} + \tfrac{u-s}{1-s}\,\psi_{s,u}(\mathbf{x}),
\label{eq:semigroup}
\end{equation}
for $s<u<t$, with $\alpha_{s,u,t} = \frac{(u-s)(1-t)}{(t-s)(1-u)}$ and $\beta_{s,u,t} = \frac{(t-u)(1-s)}{(t-s)(1-u)}$. Since $\alpha+\beta=1$, the right-hand side is a convex combination of probability vectors and therefore a valid distillation target, matched to the student by a forward KL:
\begin{equation}
\mathcal{L}_{\text{PSD}} = \mathbb{E}_{s,u,t,\,\mathbf{x}_s}\Big[\mathrm{KL}\Big(\mathrm{sg}\big[\alpha_{s,u,t}\,\hat\psi_{s,u}(\mathbf{x}_s) + \beta_{s,u,t}\,\hat\psi_{u,t}(\hat X_{s,u}(\mathbf{x}_s))\big] \,\Big\|\, \hat\psi_{s,t}(\mathbf{x}_s)\Big)\Big].
\label{eq:psd}
\end{equation}
The SemlaFlow backbone is conditioned on a single time and always predicts the clean endpoint, so our student represents only $\psi_{s,1}$. We therefore fix $t=1$, which collapses \Cref{eq:semigroup} to $\alpha=0$, $\beta=1$: the target reduces to the two-hop prediction $\psi_{u,1}(X_{s,u}(\mathbf{x}_s))$, and PSD becomes a consistency condition between a long jump to $t=1$ from $s$ and a shorter jump from the midpoint $u = (s+1)/2$. We realize the first hop from $s$ to $u$ with the frozen teacher, transporting the full molecular state---coordinates, atom types, bonds, and charges---along the linear path to $u$, and the second hop from $u$ to $t=1$ with an EMA copy of the student, which stabilizes the target as the student drifts away from its initialization. Both calls are stop-gradient; the trainable student is evaluated once, directly at $s$. We apply PSD loss on all three discrete generation heads,
\begin{equation}
\mathcal{L}_{\text{PSD}} = \sum_{c \in \{\text{atom},\,\text{bond},\,\text{charge}\}} \lambda_c\;
\mathbb{E}_{s,\,\mathbf{x}_s}\Big[\mathrm{KL}\Big(\mathrm{sg}\big[\hat\psi^{(c)}_{\text{EMA}}\big(\hat X^{\,\text{teacher}}_{s,u}(\mathbf{x}_s),\,u\big)\big] \,\Big\|\, \hat\psi^{(c)}_{\theta}(\mathbf{x}_s,\,s)\Big)\Big],
\label{eq:psd-ours}
\end{equation}
where $s$ is drawn from the same respaced schedule as the coordinate DMD term, atom and charge terms are averaged over valid nodes, and the bond term over valid node pairs. Because $\mu_{\text{fake}}$ is regressed on coordinate velocity only, categorical supervision is owned entirely by \Cref{eq:psd-ours}, and the full objective is $\mathcal{L} = \mathcal{L}_{\text{DMD}} + \lambda\,\mathcal{L}_{\text{JS}} + \mathcal{L}_{\text{PSD}}$.

\paragraph{SemlaFlow Results} We apply the aforementioned distillation recipe to SemlaFlow~\citep{irwin2025semlaflowefficient3d} and compare against SemlaFlow-respaced under a 4-NFE budget on QM9. After further tuning the $\rho$ for respacing, we choose $\rho = 7$, and \method-SemlaFlow achieves \textbf{88.51\% molecule stability}, compared to 70.08\% for SemlaFlow-respaced and 83.89\% for the 4-step \method-GeoLDM, while maintaining 89.75\% uniqueness (\Cref{tab:semlaflow}). The gain extends to the 3D-aware metrics: strain per atom drops from 5.83 to 2.15 and optimized RMSD from 0.376 to 0.252, indicating that the improvement reflects better-learned conformations rather than inflated 2D statistics. This demonstrates that \method \textbf{transfers to other molecule generation backbones} without architecture-specific modification.

\begin{table}[t]
  \centering
  \caption{\method-SemlaFlow training results on QM9 with 4 NFEs. Entries where \method-SemlaFlow surpasses SemlaFlow-respaced are in bold.}
  \label{tab:semlaflow}
  \small
  \begin{tabular}{lcc}
    \toprule
    Metric & 4-Step SemlaFlow-respaced & 4-Step \method-SemlaFlow \\
    \midrule
    Atom Stability ($\uparrow$)             & 97.68\%  & \textbf{98.98\%} \\
    Molecule Stability ($\uparrow$)         & 70.08\%  & \textbf{88.51\%} \\
    Validity ($\uparrow$)                   & 79.52\%  & \textbf{90.31\%} \\
    Connected Validity ($\uparrow$)         & 73.98\%  & \textbf{89.18\%} \\
    Uniqueness ($\uparrow$)                 & 91.74\%  & 89.75\% \\
    Novelty ($\uparrow$)                    & 96.14\%  & 90.69\% \\
    \midrule
    Energy ($\downarrow$)                   & 130.31   & \textbf{62.03} \\
    Energy per Atom ($\downarrow$)          & 7.43     & \textbf{3.51} \\
    Energy Validity ($\uparrow$)            & 75.38\%  & \textbf{89.40\%} \\
    Optimized Energy Validity ($\uparrow$)  & 75.38\%  & \textbf{89.40\%} \\
    Strain ($\downarrow$)                   & 101.83   & \textbf{37.53} \\
    Strain per Atom ($\downarrow$)          & 5.83     & \textbf{2.15} \\
    Optimized RMSD ($\downarrow$)           & 0.376    & \textbf{0.252} \\
    \bottomrule
  \end{tabular}
\end{table}

\clearpage
\section{Limitations and Broader Impact}
\label[appendix]{appendix:impact}
\subsection{Limitations}
\label[appendix]{app:limitations}

\paragraph{Uniqueness and Diversity Trade-Offs}
Despite the training efficiency and generation quality of \method, we observe a gradual
reduction in sample diversity as distillation training proceeds.
This is not unique to our method: it is a known characteristic of score-distillation
objectives, including the DMD family~\citep{yin2024improveddistributionmatchingdistillation,yin2024onestepdiffusiondistributionmatching}, which trade
stochastic uncertainty in sampling for stable, high-quality output.
In the molecular generation setting, diversity loss is particularly consequential:
a generative model used for virtual screening must cover a wide region of chemical
space to maximize the probability of discovering viable lead compounds.
We observe this effect primarily as reduced uniqueness scores over extended training,
while atom-level and molecule-level stability remain largely unaffected.

\paragraph{Generalization to other noise-sensitive molecular backbones.}
The core insight motivating \method---that certain equivariant architectures exhibit
structured sensitivity to noise in latent space that can be exploited by distribution
matching distillation and timestep respacing technique ---is not specific to the GeoLDM backbone~\citep{xu2023geometriclatentdiffusionmodels}.
We expect similar benefits to extend to other 3D molecular generative models that
operate over continuous latent representations of atomic coordinates, such as
protein--ligand co-design or conformation generation models.
Adapting \method to these settings may require re-examining the latent-space geometry
and the construction of the fake score network, particularly for architectures with
different equivariance groups (\emph{e.g.}, $\mathrm{SE}(3)$ vs.\ $\mathrm{E}(3)$)
or non-Gaussian priors.
We leave systematic benchmarking on broader molecular modalities to future work.

\subsection{Broader Impact}
\label{app:broader_impact}

This work accelerates the sampling of 3-D molecular generative models by applying
distribution matching distillation to equivariant latent diffusion models.
The primary intended application is computational drug discovery and materials design,
where generative models are used to propose novel molecular structures for subsequent
experimental validation.

\paragraph{Positive impacts.}
Reducing the inference cost of high-quality molecular generative models lowers the
computational barrier to large-scale virtual screening---a step that currently limits
the practical deployment of diffusion-based molecular design tools in academic and
resource-constrained industrial settings.
Faster generation enables higher-throughput exploration of chemical space within fixed
compute budgets, which could accelerate the early-stage identification of
therapeutically relevant molecules.
More broadly, \method is a general distillation framework applicable to any
equivariant latent diffusion model, and may therefore benefit downstream tasks such
as protein structure prediction, antibody design, and materials discovery.

\paragraph{Risks and limitations.}
As with all generative models for molecular design, there is a potential for misuse:
accelerated generation of structurally valid molecules could in principle be directed
toward the design of harmful compounds, including those with toxic, environmentally
persistent, or weaponisable properties.
We note, however, that \method operates as a distillation of an existing teacher model
and introduces no novel generative capacity beyond what is already present in the
teacher; the risk profile of \method is therefore no greater than that of the teacher
molecular diffusion model itself.
We recommend that downstream applications incorporate appropriate safety filters---for
example, standard toxicity predictors (\emph{e.g.}, hERG cardiotoxicity screening,
Ames mutagenicity tests) applied to generated structures prior to any experimental
follow-through.

\paragraph{Societal context.}
The acceleration of AI-driven drug discovery has complex distributional effects: it
may reduce time-to-drug for diseases with well-funded research programs while
potentially increasing the concentration of early-stage discovery power in
organisations with large compute resources.
We hope that the reduced inference cost introduced by \method partially counteracts
this trend by making high-quality molecular generation more accessible to the broader
research community.



\end{document}